\documentclass[11pt]{article}

\usepackage[T1]{fontenc}

\usepackage{lineno,hyperref}
\modulolinenumbers[5]

\usepackage{amsmath}
\usepackage{amssymb}
\usepackage{amsthm}
\usepackage{tikz}
\usepackage{pgfplots}
\usepackage{graphicx}
\usepackage{slashbox}
\usepackage{caption}
\usepackage{subcaption}
\usepackage{xcolor}
\usepackage{balance}
\allowdisplaybreaks
\usepackage{algorithm}
\usepackage[noend]{algcompatible}
\usepackage{mathtools}
\usepackage{subcaption}
\usepackage{hyperref}
\usepackage{comment}
\usepackage{tabularx}
\usepackage{bbm}
\usepackage{thmtools}
\usepackage{cite}

\usepackage{hyperref}
\usepackage{pgfplots}
\usepackage{graphicx}
\usepackage{caption}
\usepackage{subcaption}
\usepackage{algorithm}
\usepackage[noend]{algcompatible}
\usepackage{multirow}
\usepackage{booktabs}
\usepackage{enumitem}
\usepackage{array}
\usepackage{authblk}
\usepackage{setspace}
\usepackage{fullpage,etoolbox}
\usepackage{color, colortbl}
\definecolor{Gray}{gray}{0.9}
\newcommand{\muc}[2]{\multicolumn{#1}{c}{#2}}
\usepackage[most,breakable,many]{tcolorbox}
\tcbset{before skip=0pt,after skip=0pt} 

\allowdisplaybreaks

\renewcommand{\baselinestretch}{0.97}

\newcommand{\statee}{s}

\newcommand{\reals}{\mathbb{R}}

\newcommand{\gaussian}{\mathcal{N}}

\newcommand{\horizon}{\mathrm{K}}
\newcommand{\timeid}{k}
\newcommand{\init}{\mathcal{I}}

\newcommand{\traj}{\sigma}
\newcommand{\trajsim}{\traj^{\mathsf{sim}}}
\newcommand{\trajreal}{\traj^{\mathsf{real}}}

\newcommand{\relu}{\mathsf{ReLU}}

\newcommand{\zon}{\mathsf{Zonotope}}

\newcommand{\overallf}{\mathcal{F}}

\newtheorem{remark}{Remark}
\newtheorem{example}{Example}

\newcommand{\dist}{\trajdist^{\mathsf{real}}}
\newcommand{\distzero}{\trajdist^{\mathsf{sim}}}
\newcommand{\distR}{\trajdistR^{\mathsf{real}}}
\newcommand{\distzeroR}{\trajdistR^{\mathsf{sim}}}
\newcommand{\loss}{\mathcal{L}}

\newcommand{\trajdataset}{\mathcal{T}}
\newcommand{\traindataset}{\trajdataset^{\mathsf{trn}}}

\newcommand{\calibdataset}{\mathcal{R}^{\mathsf{calib}}}
\newcommand{\lpdataset}{\mathcal{T}^{\mathsf{LP}}}

\newcommand{\mypara}[1]{\vspace{0.3em} \noindent{\bf #1}.}  

\newcommand{\myipara}[1]{\vspace{0.3em} \noindent{\em #1}.} 

\def\relu{\mathrm{ReLU}}

\newcommand{\vx}{\mathbf{x}}

\newcommand{\navid}[1]{\textcolor{black}{#1}}

\newcommand{\distinit}{\mathcal{W}}
\newcommand{\states}{\mathcal{S}}
\newcommand{\Statee}{S}
\newcommand{\trajdist}{\mathcal{D}_{\Statee,\horizon}}
\newcommand{\trajdistR}{\mathcal{J}_{\Statee,\horizon}}

\begin{document}
\newtheorem{theorem}{Theorem}
\newtheorem{definition}{Definition}
\newtheorem{lemma}{Lemma}

\title{Statistical Reachability Analysis of Stochastic Cyber-Physical Systems under Distribution Shift}

\author[1]{Navid Hashemi}
\author[1]{Lars Lindemann}
\author[1]{Jyotirmoy V. Deshmukh}
\affil[1]{Thomas Lord Department of Computer Science, University of Southern California}

\maketitle
 
\begin{abstract}
Reachability analysis is a popular method to give safety guarantees for
stochastic cyber-physical systems (SCPSs) that takes in a symbolic description of the
system dynamics and uses set-propagation methods to compute an overapproximation
of the set of reachable states over a bounded time horizon. In this paper, we
investigate the problem of performing reachability analysis for an SCPS that
does not have a symbolic description of the dynamics, but instead is described
using a digital twin model that can be simulated to generate system trajectories.
An important challenge is that the simulator implicitly models a probability 
distribution over the  set of trajectories of the SCPS; however, it is typical 
to have a sim2real gap, i.e., the actual distribution of the trajectories in a
deployment setting may be shifted from the distribution assumed by the simulator.
We thus propose a statistical reachability analysis technique that, given a
user-provided threshold $1-\epsilon$, provides a
set that guarantees that any reachable state during deployment lies in this set with probability
not smaller than this threshold. Our method is
based on three
main steps: (1) learning a deterministic surrogate model from sampled
trajectories, (2) conducting reachability analysis over the surrogate model, and
(3) employing {\em robust conformal inference} using an additional set of sampled
trajectories to quantify the surrogate model's distribution shift with respect to
the deployed SCPS. 
To counter conservatism in reachable sets, we propose a novel method to train 
surrogate models that minimizes a quantile loss term (instead of the usual mean squared
loss), and a new method that provides tighter guarantees using conformal inference
using a normalized surrogate error. We demonstrate the effectiveness of our technique
on various case studies.
\end{abstract}

\numberwithin{equation}{section}
\renewcommand\thetheorem{\arabic{section}.\arabic{theorem}}
\renewcommand\theequation{\arabic{section}.\arabic{equation}}

\section{Introduction}
Safety-critical cyber-physical systems operate in highly dynamic and uncertain
environments. It is common to model such systems as {\em stochastic dynamical
systems} where given an initial configuration (or state) of the system, system
parameter values, and a sequence of exogenous inputs to the system, a {\em
simulator} can provide a system trajectory. Several executions of the simulator
can generate a sample distribution of the system trajectories, and such a
distribution can then be studied with the goal of analyzing safety and
performance specifications of the system. In safety verification analysis, we
are interested in checking if any system trajectory can reach an {\em unsafe}
state. A popular approach for safety verification considers only bounded-time
safety properties using (bounded-time) {\em reachability analysis}
\cite{abate2008probabilistic,abate2007computational,huang2019reachnn,dutta2019sherlock,bansal2017hamilton}.
Here, the typical assumption is that the symbolic dynamics of the simulator
(i.e. the equations it uses to provide the updated state from a previous state
and stimuli) are known. Most reachability analysis methods rely on a
deterministic description of the symbolic dynamics and use set-propagation
methods to compute a {\em flowpipe} or an overapproximation of the set of states
reachable over a specified time horizon. Other methods allow the system dynamics
to be stochastic, but rely on linearity of the dynamics to propagate
distributions over initial states/parameters to compute probabilistic reach sets
\cite{vinod2019sreachtools,vinod2017forward,adzkiya2015computational,gan2017reachability}.

% namical system to over-approximate and determine the set of reachable states,
% often referred to as \textit{'flowpipes'}, over a pre-specified time horizon.
% Reachability analysis of stochastic nonlinear dynamical systems has been a subject of extensive study in the literature 
%Most prior work adopts a \textit{model}-\textit{based} approach,  relying on a symbolic model of the 
% Extensive research has also been conducted in the field of model-free reachability analysis, which can be broadly categorized into two groups. The first group assumes that the underlying distribution of the dataset is identical to the system under study. In real-world scenarios, a distribution shift between the deployment environment and training/testing environment may also occur. The second group considers the mentioned distribution shift and relaxes the assumption of identical underlying distributions between the training/testing environment and the deployment environment.

However, for complex cyber-physical systems, dynamical models may be highly
nonlinear or hybrid with artifacts such as look-up tables, learning-enabled
components, and proprietary black-box functions making the symbolic dynamics
either unavailable, or difficult for existing (symbolic) reachability analysis
tools to analyze them. To address this issue, we pursue the idea of
\textit{model-free analysis}, where the idea is to compute reachable sets for
the system from only sampled system trajectories \cite{hashemi2023data,thorpe2019model}. The
main idea of data-driven reachability analysis in \cite{hashemi2023data}
consists of the following main steps: Step 1. Sample system trajectories based
on a user-specified distribution on a parametric set of system uncertainties
(such as the set of initial states). Step 2. Train a data-driven surrogate model
to predict the next $\horizon$ states from a given state (for example, a neural
network-based model). Step 3. Perform set-propagation-based reachability
analysis using the surrogate dynamics. Step 4. Inflate the computed flowpipe
with a surrogate error term that guarantees that any actually reached state is
within the inflated reach set with probability not smaller than a user-provided
threshold. 

There are three main challenges in this overall scheme: (1) In
\cite{hashemi2023data}, a simple training loss based on minimizing the mean
square error between the surrogate model and the actual system is used. This may
lead to the error distribution to have a heavy tail, which in turn leads to
conservatism in the inflated reach set. (2) The approach in
\cite{hashemi2023data} uses the uncertainty quantification technique of {\em
conformal inference} to construct the inflated flowpipes, but  quantifies
surrogate error per trajectory component (i.e, per state dimension and per
trajectory time-step). These per-component-wise probabilistic guarantees are
then combined using union bounding, i.e., using that $P(A \cup B) \le P(A) +
P(B)$, leading to conservatism. This is because requiring a $1-\epsilon$ probability
threshold on the inflated reach set requires stricter probability thresholds in
the conformal inference step per component, i.e., thresholds $1-\epsilon'$ with $\epsilon' =\frac{\epsilon}{n\horizon}$, where $n$ is the number
of dimensions and $\horizon$ is the number of time-steps in the trajectory. A stricter probability threshold induces a larger
uncertainty set, which implies greater conservatism. (3) The most significant
real-world challenge is that the surrogate model is usually learned based on the
trajectories sampled from the simulator, and thus distributed according to the
assumptions on stochasticity made by the simulator. However, the actual
trajectory distribution in the deployed system may change. Typically, such
distribution shifts can be quantified using divergence measures such as an
$f$-divergence or the Wasserstein distance \cite{vallender1974calculation}.

To address these challenges, we propose a robust and
efficient approach to computing probabilistic reach sets for stochastic systems,
with the following main contributions: (1) We propose novel training algorithms
to obtain surrogate models to forecast trajectories from sampled initial states
(or other model parameters). Instead of minimizing the mean square loss between
predicted trajectories and the training trajectories, we allow minimizing an
arbitrary quantile of the loss function. This provides our models with better
overall predictive performance over the entire trajectory space (i.e., over
different state dimensions and time steps). (2) Similar to
\cite{hashemi2023data}, we utilize conformal inference (CI) to quantify
prediction uncertainty. However, inspired by work in
\cite{cleaveland2023conformal}, we  compute the maximum of the weighted
residual errors to compute the nonconformity score to use with CI which has the
effect of normalizing component-wise residuals. In contrast to
\cite{cleaveland2023conformal}, which solves a linear complementarity problem to
compute these weights, we obtain these weights when training the surrogate
model using gradient descent and backpropagation. (3) Finally, to address
distribution shifts, we use techniques from robust conformal inference
\cite{cauchois2020robust}. Our analysis is motivated by \cite{zhao2023robust} and valid for all trajectory
distributions corresponding to real-world environments that are close
to the original trajectory distribution used for training the surrogate model;
here, the proximity is measured by a certain $f$-divergence metric \cite{csiszar1972class}. 

We show that our training procedure and the use of the max-based nonconformity
score noticeably enhances data efficiency and significantly improves the
conservatism in reachability  analysis. This improvement in data efficiency is
the key factor that enables us to efficiently incorporate robust conformal
inference in our reachability analysis. We empirically validate our algorithms
on challenging benchmark problems from the cyber-physical systems community
\cite{huang2022polar}, and demonstrate considerable
improvement over prior work.

\mypara{Related Work}

% In comparison, our
% approach relies solely on original sampled trajectories with no change of state
% dimension. conformal inference has gained noticeable attention in the community
% due to its data efficiency. 
\myipara{Reachability Analysis for Stochastic Systems with known Dynamics}
Reachability analysis is a widely studied topic and typically assumes access to
the system's underlying dynamics, and the proposed guarantees are valid only on
the given model dynamics. In \cite{lin2023generating}, the authors propose
DeepReach, a method using neural PDE solvers for Hamilton-Jacobi method-based
reachability analysis in high-dimensional systems. While it incorporates neural
methods for reachability analysis, it still requires access to the system
dynamics.  In \cite{alanwar2023data}, the authors identify Markovian stochastic
dynamics from data through specific parametric models, such as linear or
polynomial, followed by reachability analysis on the identified models. In
contrast, our method employs neural networks, which are not confined to
Markovian dynamics. The approach in \cite{yang2022efficient} is an algorithm
that sequentially linearizes the dynamics and uses constrained zonotopes for set
representation and computation. In \cite{bortolussi2014statistical}, the authors
develop a method utilizing Gaussian Processes and statistical techniques to
compute reachable sets of dynamical systems with uncertain initial conditions or
parameters, providing confidence bounds for the reconstruction and bounding the
reachable set with probabilistic confidence, extending to uncertain stochastic
models. 

In \cite{huang2017safety}, the authors introduce a scalable method utilizing
Fourier transforms to compute forward stochastic reach probability measures and
sets for uncontrolled {\em linear systems} with affine disturbances. Similar
approaches are explored in \cite{vinod2019sreachtools,vinod2021stochastic}
for stochastic reachability analysis of linear, potentially time-varying,
discrete-time systems. A constructive
method utilizing convex optimization to determine and compute probabilistic
reachable and invariant sets for linear discrete-time systems under stochastic
disturbances is introduced in \cite{fiacchini2021probabilistic}. We note that
most existing techniques are for systems with linear dynamics, while we permit
arbitrary stochastic dynamics. In Thorpe et al.
\cite{thorpe2021approximate}, a method utilizing conditional distribution
embeddings and random Fourier features is presented to efficiently compute
stochastic reachability safety probabilities for high-dimensional stochastic
dynamical systems without prior knowledge of system structure. We note that this work does not provide finite-data probability guarantees as we do, but asymptotically converge to the exact reachset.

% Vinod et al.\cite{vinod2021stochastic} propose a scalable
% algorithm, utilizing convex optimization, to compute a polytopic
% underapproximation of the stochastic reach set and synthesize an open-loop
%controller. Prandini et al. \cite{prandini2006stochastic}.  

\myipara{Probabilistic Guarantees and Reachability Analysis for unknown Stochastic
Systems} Recent work has studied computation of  reachable sets with probabilistic guarantees directly from data. In \cite{devonport2021data}, the
authors employ level sets of Christoffel functions
\navid{\cite{lasserre2019empirical,marx2021semi}} to achieve 
probabilistic reach sets for general nonlinear systems. \navid{Specifically, let $v_d(\vx)$
denote the vector of monomials up to degree $d$, and let $M$ denote the
empirical moment matrix obtained by computing the expected value of
$v_d(\vx)^\top v_d(\vx)$  by sampling over the set of reachable states. An
empirical inverse Christoffel function $\Lambda^{-1}(\vx)$ is then defined as
$v_d(\vx)^\top M^{-1}v_d(\vx)$. The main idea in
\cite{devonport2020data,tebjou2023data} is to empirically determine
$\Lambda^{-1}(\vx)$ and give probabilistic bounds using the volume of the actual
reachset contained in the sublevel sets of $\Lambda^{-1}(\vx)$.}
In \cite{tebjou2023data}, the authors
extend the method proposed in \cite{devonport2021data} by including
conformal inference. \navid{A key challenge of this approach is estimating
the moment matrix $M$ from data, which may not scale with increasing state
dimension $n$ and user-selected degree $d$, as the dimension of $M$ is ${n+d \choose d}$,
and the approach requires inverting $M$.}
% \navid{that makes the
% dataset ellipsoidally separable to be efficiently classified with christoffel
% function}. \navid{Like the linear classification that suffers from matrix
% inverse computation, this technique also suffer from inverse computation for the
% shape matrix in cristoffel function, that its dimension increases noticeably
% with nonlinearity of systems and dimension of the states.} \navid{ Application
% of christoffel function on data analysis has been also firstly proposed by
% research works \cite{lasserre2019empirical,marx2021semi}.} 

In \cite{devonport2020data}, the authors use a Gaussian process-based classifier
to distinguish reachable from unreachable states and approximate the reachset.
However, the approach requires adaptive sampling of initial states, which may
require solving high-dimensional optimization problems. They also propose an
interval abstraction of the reachset, which, though it provides sample
complexity bounds, can be overly conservative and computationally costly in
high-dimensional systems. The method in \cite{fisac2018general} assumes partial
knowledge of the model and leverages data to handle Lipschitz-continuous
state-dependent uncertainty; their reachability analysis combines probabilistic and worst-case analysis. Finally, the work
presented \cite{dryvr} combines simulation-guided reachability analysis
with data-driven techniques, utilizing a discrepancy function estimated from
system trajectories, which can be challenging to obtain.

\myipara{Reachability analysis for Neural Networks} Recent approaches
have tackled the challenge of determining the output range of a neural network.
These methods aim to compute an interval or a box (a vector of intervals) that
encompasses the outputs of a given neural network. Katz et al. \cite{katz2017reluplex} introduced
Reluplex, an SMT-based approach that extends the simplex algorithm to handle
ReLU constraints. Huang et al. \cite{huang2017safety} employed a
refinement-by-layer technique to verify the presence or absence of adversarial
examples in the vicinity of a specific input. Dutta et al.
\cite{dutta2019sherlock} proposed an efficient method using mixed-integer linear
programming to  compute the  range of a neural network
featuring only ReLU activation functions. Tran et al. \cite{tran2020nnv}
proposes star-sets that offer similar expressiveness as hybrid zonotopes and
are used to provide approximate and exact reachability of
feed-forward ReLU neural networks. In our setting, this method was the most
applicable.

\section{Problem statement and Preliminaries}
\label{sec:prelim}
\mypara{Notation} We use bold letters to represent vectors and vector-valued
functions, while calligraphic letters  denote sets and distributions. The set
$\left\{1,2,\cdots, n \right\}$ is denoted as $[n]$. The Minkowski sum   is indicated by $\oplus$. We use $x \sim
\mathcal{X}$ to denote that the random variable $x$ is drawn from the
distribution $\mathcal{X}$.

% For
% clarity, we illustrate the structure of a feedforward neural network (FFNN) with
% $\ell$ hidden layers as an array $[n_0,n_1,\cdots n_{\ell+1}]$, where $n_0$
% denotes the number of inputs, $n_{\ell+1}$ represents the number of outputs, and
% $n_i$, with $i\in [\ell]$, indicates the width of the $i$-th hidden layer. 

% The cardinality of a set $A$ is denoted by
% $|A|$.

%In addition, we
%make use of the notation $\lceil x \rceil$ to represent the smallest integer
%greater than the real number $x\in \reals$. 

\mypara{Stochastic Dynamical Systems} We consider discrete-time stochastic
dynamical systems. While it is typical to describe such systems using symbolic
equations that describe how the system evolves over time, we instead simply
model the system as  a {\em stochastic process}. In other words, let
$\Statee_0,\ldots,\Statee_\horizon$ be a set of $\horizon+1$ random vectors
indexed by times $0,\ldots,\horizon$. We assume that for all times $\timeid$,
each $\Statee_\timeid$ takes values from the set of states $\states \subseteq
\reals^n$. A realization of the stochastic process, or the {\em system
trajectory}  is a sequence of values $\statee_0,\ldots,\statee_\horizon$,
denoted as $\trajreal_{\statee_0}$. The joint distribution over $\Statee_0,
\ldots,\Statee_\horizon$ is called the {\em trajectory distribution} $\dist$ of
the system, and the marginal distribution of $\Statee_0$ is called the {\em
initial state distribution} $\distinit$. We assume that the initial state
distribution $\distinit$ has support over a compact set of initial states $\init$, i.e., we assume that $\distinit$ is such that $\Pr[\statee_0 \notin \init]=0$.
For example, such a stochastic dynamical system could describe a Markovian
process, where for any $\timeid \ge 1$, the distribution of $\Statee_\timeid$
only depends on the realization of $\Statee_{\timeid-1}$ and not the values
taken at any past time. However, it is worth noting that the techniques
presented in this paper can be applied to systems with non-Markovian
dynamics. 

In the rest of the paper, we largely focus on just the system trajectories, so
we abuse notation to denote $\statee_0 \stackrel{\mathcal{W}}{\sim} \init$ to
signify that $\statee_0$ is a value sampled from $\init$ using the initial state
distribution $\mathcal{W}$.\footnote{$\mathcal{W}$ is assumed to be 
 uniform or  truncated Gaussian distributed in practice. } Similarly,
$\trajreal_{\statee_0} \sim \dist$ is used to denote the sampling of a
trajectory from the trajectory distribution.

% We can consider
% $\mathcal{W}$ as the projection of distribution $\dist$ to a portion of the
% trajectory that presents the initial state. A practical choice for $\mathcal{W}$
% may involve uniform sampling of $\init$.
 
% The problem of estimating the set
% of reachable states for stochastic dynamical systems, starting from a compact
% set of initial states $\init \subset \reals^n$, has been extensively studied
% in the literature. While much of the research in this domain has employed
% model-based methodologies, this paper focuses on  \textit{'model-free
% reachability analysis'} applied to a black-box stochastic dynamical system
% denoted as $M$. 
% In this context, a random trajectory of $M$ can be represented
% as a sequence of time-stamped states, denoted as $\traj_{\statee_0} =
%The distribution for trajectories generated from
% $M$ is denoted as $\dist$, with $\traj_{\statee_0} \sim \dist$ indicating that
% $\traj_{\statee_0}$ is sampled from $\dist$. This distribution, $\dist$, is
% described by two key parameters, including a distribution over the set of
%initial states of the system and stochastic uncertainties in its dynamics. 
% We note that for a special class of systems, the trajectory dynamics might

%We refer to the system
%that we utilize to generate dataset as \textit{training environment} and the
%system under study for reachability analysis as \textit{deployment environment}.
%We denote the trajectories of training environment with $\traj_{\statee_0}^0$
%and the trajectories from deployment environment with $\traj_{\statee_0}$. 

\mypara{Quantification of Distribution Shift} In practice, we usually do not
have knowledge of the  distribution $\dist$. However, one may have access to
trajectories sampled from a distribution $\distzero$ that is ``close'' to
$\dist$, e.g., a simulator. Given a distribution $\mathcal{D}$, we use the
notation $\mathcal{P}(\mathcal{D})$ to denote a set of distributions {\em close}
to $\mathcal{D}$, where the notion of proximity is defined using a suitable
divergence measure or metric quantifying distance between 
distributions. Common examples include $f$-divergence measures
(such as KL-divergence, total variation distance) and metrics such as the
Wasserstein distance
\cite{vallender1974calculation,shafieezadeh2015distributionally}. In this paper,
we assume that $\distzero$ comes from the ambiguity set $\mathcal{P}(\distzero)$
that is centered at $\distzero$ using $f$-divergence balls around $\distzero$
\cite{shafieezadeh2015distributionally}.\footnote{Examples of $f$ include $f(z) = z \log(z)$, which induces the
KL-divergence and $f(z) = \frac{1}{2} \mid z-1 \mid$, which induces the
total variation distance.} Given a  convex function $f: \reals \to
\reals$ satisfying $f(1) = 0$ and $f(z) = +\infty$ for $z < 0$, the
$f$-divergence \cite{csiszar1972class} between the  probability distributions
$\distzero$ and $\dist$ that both have support $\mathcal{Z}$ is
\[
D_f(\dist \parallel \distzero) = \int_{\mathcal{Z}} f\left( \frac{\mathbf{d} \dist}{\mathbf{d} \distzero}\right) \mathbf{d} \distzero.
\]
\noindent Here, the argument of $f$ is the Radon-Nikodym derivative of $\distzero$ w.r.t. $\dist$.
We define the set $\mathcal{P}_{f,\tau}(\distzero)$ as
a $f$-divergence ball of radius $\tau \geq 0$ around $\distzero$ as
 \[
 \mathcal{P}_{f,\tau}(\distzero) = \left\{ \dist  \mid  D_f(\dist\parallel\distzero) \leq \tau \right\}.
 \]
The radius $\tau$ and the function $f$ are both user-specified parameters that
quantify the distribution shift between $\dist$ and $\distzero$ that we have to
account for in our reachability analysis. Specifically, we have to perform
reachability analysis for random trajectories $\trajreal_{\statee_0}\sim \dist$
for all $\dist\in  \mathcal{P}_{f,\tau}(\distzero)$.

\mypara{Conformal Inference} Conformal inference 
\cite{vovk2005algorithmic,lei2014distribution,lei2018distribution} is a
data-efficient statistical tool proposed for quantifying uncertainty,
particularly valuable for assessing the uncertainty in predictions made by
machine learning models \cite{angelopoulos2021gentle,luo2022sample}. 

Consider a set of random variables $z_1,z_2,...,z_{m+1}$ where $z_i=(x_i,y_i)
\in \reals^n \times \reals$ for $i\in [m+1]$. Assume that $z_1,z_2,...,z_{m+1}$
are independent and identically distributed (i.i.d.). Let $\mu(x_i)$ be a
predictor that estimates outputs $y_i$ from inputs $x_i$. With a pre-defined
miscoverage level  $\epsilon \in (0, 1)$, conformal inference enables computation
of a \navid{threshold $d>0$ and a }probabilistic prediction interval $C(x_{m+1})=[\mu(x_{m+1})-d,\
\mu(x_{m+1})+d] \subseteq \reals$ for $y_{m+1}$ that guarantees that
\(
\Pr[ y_{m+1} \in C(x_{m+1})] \geq 1- \epsilon
\). \navid{To compute the threshold $d$,
we reason over the {\em empirical distribution of the 
residual errors} between the predictor and the ground
truth data.
Let $R_i:=|y_i-\mu(x_i)| $ be the residual error
between $y_i$ and $\mu(x_i)$ for  $i \in[m+1] $.} Since the random variables $z_1,z_2,...,z_{m+1}$
are i.i.d., the residuals $R_1,\hdots, R_{m+1}$ are also i.i.d. If $m$ satisfies
$\ell:=\lceil (m+1)(1-\epsilon)\rceil\le m$, then we take the $\ell^{th}$ smallest error among these $m$ values which is equivalent to
\begin{equation}
\label{eq:conf_quant}
R_{1-\epsilon}^* = \text{Quantile}^c_{1-\epsilon} \left\{ R_1, \hdots, R_m, \infty \right\},
\end{equation}
i.e., the $(1-\epsilon)$-quantile over $R_1, \hdots, R_m, \infty$, see \cite{tibshirani2019conformal}. 

\navid{Conformal inference uses this  quantile to obtain the 
probability guarantee $\Pr[R_{m+1} \leq R^*_{1-\epsilon}] \geq (1-\epsilon)$, see 
\cite{tibshirani2019conformal,vovk2005algorithmic}. For the choice of $R_i:=|y_i-\mu(x_i)| $, this can be rewritten as}
\begin{equation} \label{eq:marginal_guarantee}
\navid{\Pr\left[ y_{m+1} \in [\mu(x_{m+1}) - R^*_{1-\epsilon}, \mu(x_{m+1}) + R^*_{1-\epsilon}] \right] \geq 1-\epsilon.}
\end{equation}

The guarantees in \eqref{eq:marginal_guarantee} are marginal\footnote{The
guarantees from conformal inference are marginal over all potentially sampled calibration
sets. The guarantees over some fixed calibration set can be shown to be a random variable
that has distribution $\mathbf{Beta}(\ell, m+1-\ell)$ \cite{angelopoulos2021gentle}. For example, if $m= 10^4$, we get tight probabilistic
guarantees for any $\epsilon\in(0,1)$ as the variance of the
$\mathbf{Beta}$ distribution is bounded by $2.5 \times 10^{-5}$.}, i.e., over the
randomness in $R_{m+1}, R_1, R_2, \hdots, R_m$. Note that $R^*_{1-\epsilon}$ is
a provable upper bound for the $(1-\epsilon)$-quantile\footnote{\navid{For any
$\epsilon \in (0,1)$, the $(1-\epsilon)$-quantile of a random variable $R$ is defined as
$\inf\{z\in\reals | \text{Pr}[R\le z]\ge 1-\epsilon\}$.}} \navid{of the  error
distribution}.

% The conditional probability (regarding
% the fixed calibration dataset) $\Pr[R\le R^*]$ is known to be a random variable
% that follows the distribution $\mathbf{Beta}(\ell, m+1-\ell)$, that is a
% continuous distribution with mean value $\ell/(m+1)$ and variance
% $[\ell(m+1-\ell)]/[ (m+1)^2(m+2) ]$ and assuming $\ell=\lceil (m+1)(1-\epsilon)\rceil$ is centered around $1-{\epsilon}$ with
% decreasing variance as $m$ increases. As an example we can conclude, for all, $\epsilon \in(0,1)$ the choice $m= 10^4$ implies the variance is upper bounded with $2.5 \times 10^{-5}$.
% }.

% = \mathsf{residual}(U; \mu)$} 
% = \mathsf{residual}(V;\mu)$

\mypara{Robust Conformal Inference}  Unlike conformal inference, which assumes
the data-point $z_{m+1}$ is sampled from the same distribution as the calibration samples
$z_i, i \in [m]$, robust conformal inference relaxes this assumption and allows
$z_{m+1}$ to be sampled from a different distribution. Let us denote the
distribution of $z_i$ for $ i \in [m]$ as $U$ and the distribution of $z_{m+1}$ as
$V$. As illustrated before, the residual $R_i$ is a distribution and defined as a function of 
$z_i$. Let us denote the distribution of
$R_i$ for $i\in[m]$ with $P$ and the distribution of $R_{m+1}$ with $Q$. Further, assume $Q$
is in $\mathcal{P}_{f,\tau}(P)$. Utilizing the results
from \cite{cauchois2020robust} that assumes the
distribution of residual $R_{m+1}$ is within a $f$-divergence ball of
the distributions for $R_1, \hdots , R_m$ with radius $\tau \geq 0$, for the miscoverage
level $\epsilon \in (0,1)$, we obtain:
\[
\begin{array}{l}
\Pr[R_{m+1} \leq R^*_{\mathbf{1-\epsilon,\tau}}] \geq 1-\epsilon
\end{array}
\]
where $R^*_{1-\epsilon,\tau} = \text{Quantile}^c_{(1-\bar{\epsilon})} \left\{
R_1, \hdots, R_m , \infty \right\}$  is a \emph{robust}
$(1-\epsilon)$-quantile that is equivalent to the
$(1-\bar{\epsilon})$-quantile. We refer to $\bar{\epsilon}$ as the adjusted
miscoverage level which is computed as $\bar{\epsilon} = 1- g^{-1}_{f,
\tau}(1-\epsilon_m)$ where $\epsilon_m$ is obtained as the solution of a series
of convex optimizations problems as\footnote{\navid{Following \cite{cauchois2020robust},
Lemma A.2., we note that $g_{f,\tau}$ is related to the worst-case CDF of any distribution
with at most $\tau$ distribution shift, and $g^{-1}$ is related to the inverse worst-case CDF.}}:
\begin{equation}
\begin{array}{@{\hspace{-0.3em}}l@{\hspace{0.2em}}c@{\hspace{0.2em}}l}\label{eq:firstpart}
    \epsilon_m & = & 1-g_{f,\tau}\left( \left(1 + \frac{1}{m} \right) g^{-1}_{f,
    \tau}(1-\epsilon) \right), \\ g_{f,\tau}(\beta) & = & \inf \left\{ z \in
    [0,1]\,\middle|\,\beta
    f\!\left(\!\frac{z}{\beta}\!\right)+(1\!-\!\beta)f\!\left(\!\frac{1\!-\!z\!}{1\!-\!\beta}\!\right)
    \leq \tau \right\}\\ g^{-1}_{f,\tau}(\gamma) & = & \sup \left\{ \beta \in
    (0,1)\,\middle|\,g_{f,\tau}(\beta) \leq \gamma \right\} \end{array}
\end{equation}

Computation of $g_{f,\tau}$ and $g^{-1}_{f,\tau}$ is efficient since they are
both solutions to one dimensional convex optimization and therefore admit
efficient binary search procedures. In some cases, we have also access to a
closed form solution \cite{cauchois2020robust}. 
% for example $f(z) = \frac{1}{2}(z-1)^2$ gives,
% \begin{equation}\label{eq:secondpart}
% g_{f,\tau}(\beta) = \max \left( 0, \beta - \sqrt{\tau \beta(1-\beta)}\right), \qquad 
% g_{f,\tau}^{-1}(\gamma) =  \frac{\tau + 2\gamma + \sqrt{\tau(\tau+4\gamma -4\gamma^2)}}{2\tau+2},
% \end{equation}
% % \[
% % \begin{aligned}
% % &g_{f,\tau}(\beta) = \max \left( 0, \beta - \sqrt{2 \tau \beta(1-\beta)}\right), \\
% % &g_{f,\tau}^{-1}(\gamma) =  \frac{\tau + \gamma + \sqrt{\tau(\tau+2\gamma -2\gamma^2)}}{2\tau+1},
% % \end{aligned}
% % \]
% or in another word, for $\alpha \in (0,1], \ \tau \in[0,\infty]$, we have;
% \[
% \tilde{\alpha} = 1 - (1 +\frac{1}{m})( \frac{\tau + 2(1-\alpha) + \sqrt{\tau(\tau+4(1-\alpha) -4(1-\alpha)^2)}}{2\tau+2} )
% \]
% % \[
% % \tilde{\alpha} = 1 - (1 +\frac{1}{m})( \frac{\tau + 1-\alpha + \sqrt{\tau(\tau+2(1-\alpha) -2(1-\alpha)^2)}}{2\tau+1} )
% % \]
\begin{example}
For the total variation, $f(z) = \frac{1}{2}|z-1|$, we have $g_{f,\tau}(\beta) = \max \left( 0, \beta - \tau \right)$, $g_{f,\tau}^{-1}(\gamma) =  \gamma+ \tau,\ \gamma \in (0, 1- \tau)$.
This implies that given radius $\tau \in [0,1]$ an adjusted miscoverage level $\bar{\epsilon}$ is infeasible if $\epsilon \leq \tau$, and 
$\bar{\epsilon}$ is computed as: 
\begin{equation}\label{eq:linearrule}
\bar{\epsilon} = 1 - \left(1+\frac{1}{m}\right)\left(1-\epsilon + \tau\right), \  \epsilon \in (\tau,\ 1] , \tau \in [0, 1]
\end{equation} 
\end{example}
% $\horizon$-step trajectories 
% with the distribution $\trajsim_{\statee_0} \sim
% \distzero,\ \Pr[\statee_0 \notin \init]=0$. Consider 

% \begin{proposition}\label{lem:distrule}
%     Assuming $D_f(\dist \parallel \distzero) = \taureal < \tau$, then the coverage level of robust $\delta$-quantile $R^*_{\delta,\tau}$ on $\dist \in \mathcal{P}_{f,\tau}(\distzero)$ is $\deltareal>\bar{\delta}-\tau-1/m$, where $m$ is the number of recorded data, $\bar{\delta} = 1-\bar{\epsilon}$ and $\bar{\epsilon}$ is adjusted miscoverage level.
% \end{proposition}
% \begin{proof}
%     Since $D_f(\dist \parallel \distzero) = \taureal<\tau$, we can also claim $\dist \in \mathcal{P}_{f,\taureal}(\distzero)$. In this case, we set $\deltareal = \delta+\tau-\taureal$ and we compute for its robust $\deltareal$-quantile. As a direct result of this choice of $\deltareal$ and the equation \eqref{eq:linearrule}, this selection results in $R^*_{\deltareal,\taureal} = R^*_{\delta,\tau}$ since the adjusted miscoverage level $\bar{\epsilon}$ will be the same. This implies the coverage level of $R^*_{\delta,\tau}$ on $\dist$ is $\deltareal$ that is apparently larger than $\delta$. In addition, based on the equation \eqref{eq:linearrule} we have $\delta=\bar{\delta}-\tau - (\delta+\tau)/m$. Since $\delta+\tau<1$ we can claim $\delta>\bar{\delta}-\tau - 1/m$ or in other words, $\deltareal>\bar{\delta}-\tau - 1/m$.    
% \end{proof}

\mypara{Problem Definition} We are given a black-box stochastic dynamical system
as the training environment with the trajectory distribution $\distzero$. We assume that when this system is deployed
in the real world, the trajectories satisfy $\trajreal_{\statee_0}
\sim \dist \in \mathcal{P}_{f,\tau}(\distzero)$. Given a user-specified failure probability
$\varepsilon \in(0,1)$ and an i.i.d. dataset of trajectories sampled from 
$\distzero$, the problem is to obtain a probabilistically guaranteed flowpipe
$X$ that  contains
$\trajreal_{\statee_0} \sim \dist$ for all $\dist \in \mathcal{P}_{f,\tau}(\distzero)$ with a confidence  of 
$1-\varepsilon$. Formally,
\begin{equation}\label{eq:verific_prob}
\left.
\begin{aligned}
&\statee_0 \stackrel{\mathcal{W}}{\sim} \init, \\
&\trajreal_{\statee_0} \sim \dist\in \mathcal{P}_{f,\tau}(\distzero)
\end{aligned}
\right\}\ \implies \Pr\left[ \trajreal_{\statee_0} \in X \right]  \geq  1-\varepsilon
\end{equation}
In other words, we are interested in computing a probabilistically guaranteed flowpipe $X$ from a set of trajectories collected from $ \distzero$ so that $X$ is valid 
for all trajectories $\dist\in \mathcal{P}_{f,\tau}(\distzero)$, i.e., despite a potential distribution shift.  

% The flowpipe is also
% associated with a confidence level of $\Delta = 1-\varepsilon$.
\section{Learning A Surrogate Model Suitable for Probabilistic Reachability Analysis}
\label{sec:surrogate}
% We partition this trajectory dataset into two distinct sets, denoted as
% $D_{\text{train}}$ and $D_{\text{stat}}$. Specifically, $D_{\text{train}}$
% serves as our training dataset and is employed for the purpose of training a
%surrogate model.

% we design our training process such that a given confidence probability of $\delta \in (0, 1)$, the $\bar{\delta}$-quantile of the residuals is to be minimized. We provide more detail on this technique later in the paper.}

As we do not have access to the system dynamics in symbolic
form,  our approach to characterize the trajectory distribution is to use a
predictor,  called the {\em surrogate model}. 
\navid{
\begin{definition}
    A surrogate model $\overallf: \mathcal{X} 
    \times \Theta \to \mathcal{Y}$ is a function that approximates a given function $f:
    \mathcal{X} \to \mathcal{Y}$. Let $d_\mathcal{Y}$ be some metric
    on $\mathcal{Y}$, then the surrogate model guarantees that for some value of
    $\theta \in \Theta$, and for any $x$ sampled from a distribution over
    $\mathcal{X}$, the induced distribution over the random variable
    $d_{\mathcal{Y}}(\overallf(x;\theta),f(x))$ has good approximation properties, such as bounds on the
    moments of the distribution (e.g. mean value) or bounds on the quantile of
    the distribution.
\end{definition} }

\navid{In our setting, the set $\mathcal{X}$ is the set of states $\states$ with the distribution over $\mathcal{X}$ being $\distzero$ and
$\mathcal{Y}$ is the set of $K$-step trajectories $\states^{\horizon}$, i.e., $\overallf$
maps a given initial state (or an uncertain model parameter) to the predicted
$\horizon$-step trajectory of the system. The metric $d_{\mathcal{Y}}$ can be
any metric on the trajectory space.} 
One example surrogate model is a
feedforward neural network (NN) with $n$ inputs and $\horizon n$ outputs,
represented as $\bar{\traj}_{\statee_0} = \overallf(\statee_0 ; \theta)$ where
$\theta$ is the set of  trainable parameters. To train the surrogate
model, we need to define a specific residual error between a set of sampled
trajectories and those predicted by the model. While most surrogate models are
trained using the cumulative squared loss across a training dataset
\cite{james1992estimation}, we consider a loss function that helps us reduce
conservatism in computing the probabilistic reach set of the system.

% e begin by training a neural network surrogate model $\overallf: \reals^n \to \reals^{n(\horizon+1)}$ on the training dataset $\traindataset$. This model

\mypara{Training a Lipschitz-bounded NN based surrogate model} \navid{Training is
a procedure to identify the parameter value $\theta$ which makes the surrogate
model a good approximation; we use a data-driven method to train the surrogate
by sampling $\horizon$-step trajectories from the simulator of the original 
model.}
%where
%$\trajsim_{\statee_0} \sim \distzero,\ \statee_0
%\stackrel{\mathcal{W}^{\mathsf{sim}}}{\sim} \init$. 
We call this dataset $\traindataset$. The
surrogate model predicts the trajectory $\trajsim_{\statee_{0}}$ starting from an initial
state sampled from $\statee_0 \stackrel{\mathcal{W}}{\sim} \init$. We
denote the predicted trajectory $\bar{\traj}_{\statee_{0}}$
corresponding to $\trajsim_{\statee_{0}}$ as:
% $$
% \bar{\traj}_{\statee_{0}} = \overallf(\statee_{0}) = 
% \left[\begin{array}{clll} \statee_0^\top &\mathsf{F}^1(\statee_0) \!\!& \cdots & \!\! \mathsf{F}^n(\statee_0)\  \cdots \\
% & \!\! \ \mathsf{F}^{(n-1)\horizon}(\statee_0) \!\! & \cdots & \!\! \mathsf{F}^{n\horizon}(\statee_0)\end{array}\right]^\top
% $$
\[
\begin{aligned}
&\bar{\traj}_{\statee_{0}} = [ \statee_0^\top,\ \overallf(\statee_{0}\ ; \theta)],\ \   \text{where}, \overallf(\statee_{0}\ ; \theta) = \\
& \ \ \left[\mathsf{F}^1(\statee_0), \cdots, \mathsf{F}^n(\statee_0),  \cdots, 
\mathsf{F}^{(\horizon-1)n+1}(\statee_0), \cdots, \mathsf{F}^{n\horizon}(\statee_0)\right]^\top.
% \left[
% \begin{array}{lll}
% \mathsf{F}^1(\statee_0) & \cdots & \mathsf{F}^n(\statee_0)  % \cdots \\
%\cdots \\
%\mathsf{F}^{(n-1)\horizon}(\statee_0) & \cdots & \mathsf{F}%^{n\horizon}(\statee_0)
%\end{array}\right]^\top
\end{aligned}
\]
Here, $\mathsf{F}^{(i-1)n+r}(\statee_0)$ is the $r^{th}$ state component at the $i^{th}$ time-step in the trajectory. In other words, we stack the dimension and time in the trajectory into a single vector\footnote{The main advantage of training the trajectory as a long vector in one shot is that this approach eliminates the problem of compounding errors in time series prediction; however, this comes with higher training runtimes.}.
We remark that a trained surrogate model with a non-restricted Lipschitz constant is problematic for reachability
analysis, as approximation errors can get uncontrollably magnified resulting in trivial bounds. \navid{As a result, we use techniques from \cite{gouk2021regularisation} to penalize the Lipschitz constant of the trained NN over the course of the training process}.

% \footnote{ A common reason for generation of such a large reach sets is the presence of adversarial examples.} which in turn, can produce flowpipes with boundaries that are trivial and relatively uninformative. This motivates us to restrict ourselves to the robust techniques like Lipschitz bounded training algorithms.}
% \end{remark}

%Recent advances in the literature have demonstrated successful approaches for obtaining precise bounds in the reachability analysis of $\relu$ neural networks using polyhedral sets. \navidd{Therefore, although we are not restricted to $\relu$ neural networks, but the accuracy of these techniques motivates our focus on $\relu$ activation functions in training neural networks as surrogate models}.

% \subsection{Residual Error}
\mypara{Residual Error}
For training neural network surrogate models, a common practice is to minimize a loss function, representing the difference between the trajectory predicted by the surrogate model and the actual trajectory. To formulate this difference, we formally define the notion of the residual error as follows.

\begin{definition}[Residual Error] 
Let
$e_i\in\reals^n$ denote the $i$-th basis vector of $\reals^n$.
For a trajectory $(\statee_0,
\trajsim_{\statee_0})$ with $\trajsim_{\statee_0}$ sampled from 
$\distzero$, and $\statee_0 \stackrel{\mathcal{W}}{\sim} \init$, we define:
\begin{equation}
\label{eq:compwise_residual}
R^j =
\left|e_{j+n}^\top \trajsim_{\statee_0} - \mathsf{F}^j(\statee_0)\right|,\quad j \in[n\horizon].
\end{equation}
Note that $R^j$ is a non-negative prediction error between the
$(j+n)^{th}$ component\footnote{There is offset of $n$ as the first $n$
components of $\trajsim_{\statee_0}$ are the initial state.} of
$\trajsim_{\statee_0}$ and its prediction $\mathsf{F}^j(\statee_0) , j \in [n\horizon]$. 
The trajectory residual $R$ is then
defined as the largest among all scaled, component-wise prediction errors with scaling factors $\alpha_j >0, j\in[n\horizon]$, i.e., $R$ is defined as 
\begin{equation}\label{eq:defresidual}
R  = \max \left( \alpha_1 R^1, \alpha_2 R^2, \cdots, \alpha_{n\horizon}R^{n\horizon}\right).  
\end{equation}
\end{definition}

% is a specific selection for the function $R =
% \mathsf{residual}(\trajsim_{\statee_0},\statee_0; \overallf)$, that 

Note that this definition is inspired by \cite{cleaveland2023conformal}\footnote{In this definition, we consider component-wise
residual for $R^j$ instead of a state-wise residual as the component
$e_{j+n}^\top \trajsim_{\statee_0}$ in $\trajsim_{\statee_0}$ may represent
different quantities like velocity or position. State-wise residuals may lead to
a higher level of conservatism in robust conformal inference, as the magnitude of
error in different components of a state may be noticeably different.}. Compared
to \cite{hashemi2023data}, utilizing the maximum of weighted errors obviates
the need to union bound component-wise probability guarantees to obtain a
trajectory-level guarantee. Let $R_i= \max(\alpha_1 R_i^1,\  \alpha_2 R_i^2,\hdots, \alpha_{n\horizon}
R_i^{n\horizon})$ for $ i \in [|\traindataset|]$ denote the trajectory residual as in \eqref{eq:defresidual} for the training dataset $\traindataset$.

\mypara{Training using a $\bar{\delta}$-quantile loss}
\navid{Let $\bar{\delta} = 1-\bar{\epsilon}$ where $\bar{\epsilon}$ is the adjusted miscoverage level as defined previously.} The ultimate goal from training a surrogate model is to achieve a higher level
of accuracy in our reachability analysis. The mean squared error (MSE) loss
function is a popular choice to train surrogate models; however, \navid{we later show that our proposed flowpipe is generated based on the quantile of the trajectory residual error. Although the MSE loss function is popular and efficient, it may result in a heavy tailed distribution for the residual error which can imply a noticeably larger quantile and result in conservative flowpipes.}
% consider a
% scenario where large errors for some trajectory components are subdued by
% several small errors due to averaging. Here, though the MSE loss is small, when
% we later use conformal inference to quantify component-wise error, the large
% residuals can lead to excess conservatism when computing the probabilistic reach
% sets.
Thus, to improve overall statistical guarantees, we are
interested in minimizing the $\bar{\delta}$-quantile of the trajectory-wise residuals,
for an appropriate $\bar{\delta}\in [0,1)$; towards that end, we add a new trainable parameter
$q$. We can also setup the training process such that the scaling factors 
$\alpha_1,\ldots,\alpha_{n\horizon}$ become decision variables for the optimization
problem. Thus, the set of trainable parameters includes the NN parameters 
$\theta$, the scaling factors $\alpha_1,\cdots, \alpha_{n\horizon}$ and the parameter $q$ that approximates the $\bar{\delta}$-quantile of the residual loss. We define two loss functions:
\begin{enumerate}[wide, labelwidth=!, labelindent=0pt]

 \item The first loss function $\loss_1$ is to set the trainable parameter $q$
 as the $\bar{\delta}$-quantile of trajectory-wise residuals. This loss function is
 inspired from literature on quantile regression \cite{koenker2005quantile}, and it is a well-known result that  minimizing
 this function yields $q$ to be the $\bar{\delta}$-quantile of $R_1,\ldots,R_{|\traindataset|}$. \navid{ Thus, given a batch of training data points of size $M < |\traindataset|$, let
\begin{equation}\label{eq:L1}
\loss_1 = \sum_{i=1}^{M} \bar{\delta} \ \relu(R_i - q) + (1-\bar{\delta})\  \relu(q-R_i) . 
\end{equation}
}
\item  Assuming $q$ as the $\bar{\delta}$-quantile of the i.i.d. residuals
$R_i$, we let the  second loss function
$\loss_2$  minimize 
\begin{equation}\label{eq:L2}
\loss_2 = q \left( \frac{1}{\alpha_1} + \frac{1}{\alpha_2} + \cdots + \frac{1}{\alpha_{n\horizon}} \right).
\end{equation}
This is motivated by the fact that, for all $j\in[n\horizon]$, $R_i^j \leq
R_i/\alpha_{j}$ by the definition of $R_i$. Thus, the sum of errors over the trajectory components is
upper bounded by:
\begin{equation}\label{eq:UB}
\mathbf{UB}_i = R_i \left( \frac{1}{\alpha_1} + \frac{1}{\alpha_2} + \cdots + \frac{1}{\alpha_{n\horizon}} \right),
\end{equation}
and the $\bar{\delta}$-quantile of $\mathbf{UB}_i, i \in [|\traindataset|]$ is nothing but
$\loss_2$\footnote{In case we replace $\loss_2$ with $q$, the trivial solution
for scaling factors is $\alpha_j = 0 , j\in[n\horizon]$. Therefore, the proposed
secondary loss function $\loss_2$ also results in avoiding the trivial
solution for scaling factors.}.

    % The surface area of the inflating zonotope is given by the formulation:
    % \[
    %  S = R^* \left( \frac{1}{\alpha_1} + \frac{1}{\alpha_2} + \cdots + \frac{1}{\alpha_{n\horizon}} \right).
    %  \]
    % Here, \(R^*\) represents the conformalized \(\delta\) quantile of residuals $R_i = \max( \alpha_1 R_i^1,\  \alpha_2 R_i^2,\hdots, \alpha_{n\horizon} R_i^{n\horizon})$, where \(i \in [L]\) from the calibration dataset. It's important to note that the calibration dataset cannot be utilized in the training algorithm, therefore, we can not accurately compute \(R^*\) with the training dataset, i.e. \(i \in [L_2]\). On the other hand,
\end{enumerate}         
Therefore, we define the loss function as,
\begin{equation}\label{eq:mainloss}
 \loss = c\loss_1+ \loss_2, 
 \end{equation}
 where $c$ is a large number that penalizes $\loss_1$ to make sure that $q$ serves as a good approximation for the $\bar{\delta}$-quantile. The training itself uses standard back-propagation methods for computing the gradient of the loss function, and uses stochastic gradient descent to train the surrogate model.

\mypara{Properties of surrogate model} \navid{
We pick neural networks (NN) as surrogate models due to their computational
advantages and the ability to fit arbitrary
nonlinear functions with low effort in tuning hyper-parameters. We note that the input layer of the NN is
always of size $n$ (the state dimension), and the output layer
is of size $n\horizon$ (the dimension of the predicted 
trajectory over $K$ time-steps.) In our experiments, we choose NNs with 2-3 hidden layers for which we observed good results; picking more hidden layers will give better
training accuracy, but may cause overfitting. In each hidden
layer we pick an increasing number of neurons between $n$ and
$n\horizon$.}

% Given a trajectory
%  $\trajsim_{\statee_0}$ from training dataset, the loss function generates the
%  residual errors $R^j, j\in [n\horizon]$ from \eqref{eq:compwise_residual} and
%  using the trainable parameters $\alpha_j, j\in[n\horizon]$ it generates the
%  residual $R$ from \eqref{eq:defresidual}. At each iteration of backpropagation,
%  given a mini-bach of dataset $(\statee_{0,i}, \trajsim_{\statee_{0,i}}) ,
%  i\in[M]$, the loss function utilizes the residuals $R_i, i\in[M]$ to generate
%  $\loss_1$ and $\loss_2$ from \eqref{eq:L1} and \eqref{eq:L2}. The training
%  algorithm also uses $\loss = c\loss_1 + \loss_2$  for backward computation with
%  respect to $(\theta, \alpha_1, ...,\alpha_{n\horizon}, q)$ and updates them via
%  SGD that is also equipped with some weight regularization tools
%  \cite{gouk2021regularisation} to penalize the Lipstchitz constant of NN with
%  respect to the input datapoints $\statee_{0,i} , i\in [M]$.
%  
\section{Scalable Data-Driven Reachability Analysis}
\label{sec:verification}
% In contrast, $\calibdataset$ constitutes our dataset,
% utilized to construct the calibration dataset essential for deriving statistical
% properties. Mathematically, we represent $\calibdataset$ as
% $\calibdataset=\{\statee_{0,i}, \trajsim_{\statee_0,i}\}_{i=1}^L$, where $L$
% denotes the number of calibration trajectories. We also denote the number of
% train trajectories with $|\traindataset |$. 

% It is worth noting that although the data points
% within a single trajectory may not be independently and identically distributed
% (i.i.d.) the trajectory $\trajsim_{\statee_0}$ can be regarded as
% i.i.d. samples within the $\reals^{n(\horizon+1)}$-space. This observation
% plays a pivotal role in our subsequent quantification of uncertainty using
% conformal inference.

In this section, we show how we can compute a robust probabilistically
guaranteed reach set or flowpipe $X\subset \reals^{n(\horizon+1)}$ for a
stochastic dynamical system. Given a miscoverage level $\epsilon$, we wish to be
at least $(1-\epsilon)$-confident about the reach-set that we compute. For
brevity, we introduce $\delta = (1-\epsilon)$. In the procedure that we
describe, we compute a probabilistically guaranteed $\delta$-confident flowpipe, defined as follows: 
% \footnote{We later show that our
% procedure guarantees that $\Delta > \delta$, and the relationship between $\Delta$ and $\delta$ will be clarified later in this section.}
\begin{definition}[$\delta$-Confident Flowpipe] For a given confidence
probability $\delta\in (0,\ 1)$, a distribution $\distzero$, the radius $\tau$,
and a $f$-divergence ball $\mathcal{P}_{f,\tau} (\distzero)$, we say that $X
\subseteq \reals^{n(\horizon+1)}$ is a $\delta$-confident flowpipe if  we have
% \begin{equation}\label{eq:confreg}
$\Pr[\trajreal_{\statee_0} \in X ] \geq \delta$ for any
random trajectory $\trajreal_{\statee_0} \sim \dist \in \mathcal{P}_{f,\tau}
(\distzero)$ with $\statee_0 \stackrel{\mathcal{W}}{\sim} \init$.
%\end{equation}
\end{definition}
Our objective is to compute $X$ while being limited to sample trajectories from the training environment $\distzero$. We will demonstrate that we can compute $X$ with formal probabilistic guarantees by combining reachability analysis on the surrogate model trained from $\traindataset$ and error analysis on this model via robust conformal inference. 

\mypara{Deterministic Reachsets for the Surrogate Models} Using the surrogate model
from Section \ref{sec:surrogate}, we  show how to perform
deterministic reachability analysis to get {\em surrogate flowpipes}.
\begin{definition}[Surrogate flowpipe] 
The surrogate flowpipe $\bar{X}\subset \reals^{n(\horizon+1)}$ is defined as a 
superset of the image of $\overallf(\init\ ;\theta)$. Formally, for all 
$\statee_0\in \init$, we need that
$[\statee_0^\top,\ \overallf(\statee_0\ ;\theta) ] \in \bar{X}$.   
\end{definition}
% \textcolor{blue}{Since the surrogate model predicts the overall trajectory a long vector, we do not face the famous problem of compounding error for time series  }This prediction task presents the challenge of accumulating errors, leading to unexpected deviation in the predicted trajectory from the actual system trajectory. 

Thus, to compute the surrogate flowpipe, we essentially need to compute the
image of $\init$ w.r.t. the $\overallf$. This can be accomplished by
performing reachability analysis for neural networks, e.g., using tools such as
\cite{tran2019star,tran2020nnv,zhang2018efficient,huang2019reachnn}.

% \navidd{In other words}, by partitioning the set of
% initial conditions $\init$ into $N$ sub-partitions and performing reachability
% analysis on each sub-region in parallel, we achieve noticeable improvement in
% computational efficiency and accuracy of bounds for both exact-star and
% approx-star techniques.

\mypara{Robust $\delta$-Confident Flowpipes}
In spite of training the surrogate model to maximize prediction accuracy, it is
still possible that a predicted trajectory is not accurate, especially when predicting
the system trajectory from a previously unseen initial state. \navid{Note also that we trained the surrogate model on trajectory data from $\distzero$. We thus cannot expect the predictor to always perform well on trajectories drawn from $\dist$.} We now show how
to quantify this prediction uncertainty using robust conformal inference. To do so,
we first sample an i.i.d. set of trajectories from the training environment $\distzero$, 
which we again denote as the calibration dataset.

\begin{definition}[Calibration Dataset] The calibration dataset $\calibdataset$
is defined as:
\[
\calibdataset = 
\left\{ 
    \left(\statee_{0,i}, R_i \right) \middle|
    \begin{array}{l}
        \statee_{0,i} \stackrel{\mathcal{W}}{\sim} \init, \, \trajsim_{\statee_{0,i}} \sim \distzero,  \  R_i = \max\left( \alpha_1 R_i^1, \cdots, \alpha_{n\horizon}R_i^{n\horizon}\right)
    \end{array}
\right\}.
\]
\noindent Here, $\trajsim_{\statee_{0,i}}$ refers to the trajectory
starting at the $i^{th}$ initial state sampled from $\mathcal{W}$ and the
resulting trajectory  from $\distzero$,
and $R_i^j$ is as defined in equation~\eqref{eq:compwise_residual}.
\end{definition}

\begin{remark}
It is worth noting that although the data points within a single trajectory may
not be i.i.d., the trajectory
$\trajsim_{\statee_0}$ can be treated as an i.i.d. random vector in the
$\reals^{n(\horizon+1)}$-space, and subsequently the residuals are also i.i.d. This
is crucial to apply robust conformal inference, which requires that the calibration set
is exchangeable (a weaker form of i.i.d.).
\end{remark}

%     \item The distribution shift between the training environment and the deployment environment is also another potential source for the prediction error.
% \end{enumerate} 
% To give probabilistic bounds on this error, we  utilize robust conformal inference that also considers the errors generated from distributional shift.

% We have already defined the residual $R$ for a sampled % trajectory in section
% \ref{sec:surrogate}. Let us now compute this residual for all trajectories from
% $\calibdataset$, i.e. for $\trajsim_{\statee_{0,i}}, i \in [L]$.

% Our access to the training environment is limited to a finite number of pre-recorded trajectories, with $\trajsim_{\statee_{0, i}} \sim \distzero, \ i \in [L]$ which provides us residuals with induced distributions $R_i \sim \distzeroR$ where, 

Let $\distzeroR$ be the distribution over trajectory-wise residuals for trajectories from $\trajsim_{\statee_0} \sim \distzero$. However, we wish to
get information about the trajectory-wise residual $R$ for a
trajectory sampled from  $\dist \in
\mathcal{P}_{f,\tau}(\distzero)$. Let the distribution of $R$ 
induced by $\dist$ be denoted by $\distR$.
% \begin{lemma}\label{rem:dpineq} [Theorem 4.2 in \cite{wu2019information}]
% % Consider a channel that produces $R$ for the trajectory $\sigma_{\statee_0}$,
% % and surrogate model $\overallf$, based on the law proposed in
% % \eqref{eq:defresidual}. 
% Let $R$ be sampled from the induced distribution $\distzeroR$ (corresponding to
% trajectories $\trajsim_{\statee_0} \sim \distzero$), and let $\distR$ be the
% induced distribution of $R$ (corresponding to trajectories
% $\trajreal_{\statee_0} \sim \dist$), then if $f(1)=0$, $f$ is convex, and $f$ is
% strictly convex at $1$, then for any $f$-divergence $D_f$,
% \[
% D_f(\distzeroR \parallel \distR) \le D_f(\distzero \parallel \dist )
% \]
% \end{lemma}
% \noindent This lemma is also known as the data processing inequality.
\navid{As a direct result from the data processing inequality \cite{beaudry2011intuitive}, the distribution shift between $\dist$ and $\distzero$ is  larger than the distribution shift between $\distR$ and $\distzeroR$} 
% regardless of any accurate knowledge about $\distzero$, given that
% $\dist \in \mathcal{P}_{f,\tau}(\distzero)$, there is a radius
% $\rho \geq 0$ such that, $\distR \in \mathcal{P}_{f,\rho}(\distzeroR)$, 
% Since, the
% trajectory distributions $\distzero$ and $\dist \in
% \mathcal{P}_{f,\tau}(\distzero)$ are both continuous, the function $f$ satisfies
% the condition proposed in Lemma \ref{rem:dpineq}, and the residual is also a
% measurable function from the sampled trajectories, 
%  where $\rho \leq
% \tau$. Therefore, the radius $\tau$ that is specified for the divergence between
% trajectory distributions is also a valid upper bound for the distribution shift
% between residual distributions $\distzeroR, \distR$, and 
so that we have $\distR
\in \mathcal{P}_{f,\tau}(\distzeroR)$. 

Knowing that $\distR \in \mathcal{P}_{f,\tau}(\distzeroR)$, we can utilize robust
conformal inference in \cite{cauchois2020robust} to find a guaranteed upper
bound for the $\delta$-quantile of $R$. We call this guaranteed upper
bound as robust conformalized $\delta$-quantile, and we denote it with $R^*_{\delta,\tau}$,
where, $\Pr[R \leq R^*_{\delta,\tau}] \geq \delta$. Specifically, we utilize equation
\eqref{eq:firstpart} to compute $R^*_{\delta,\tau}$ from the calibration dataset $\calibdataset$.

% This is important to note that this methodology only requires access to $\calibdataset, f, \rho$ and does not require accurate knowledge of distributions $\distzero, \dist$. The ultimate objective is to compute a $\Delta$-confident flowpipe $X$ for trajectories $\trajreal_{\statee_0} \sim \dist \in \mathcal{P}_{f, \rho}(\distzero)$ with the initial state $\Pr[ \statee_0 \notin \init]=0$. To achieve this, we compute $R^*$ from robust conformal inference, and utilize it for the proposition of the \navidd{following result.

Next we show that our definition of residual error introduced in~\eqref{eq:defresidual} allows us to use a single trajectory-wise
nonconformity score for applying robust conformal inference (instead 
of the component-wise conformal inference as in \cite{hashemi2023data}).

\begin{lemma}\label{lem:maxcontribution}
    Assume $R^*_{\delta,\tau}$ is the $\bar{\delta}$-quantile  computed over the residuals $R_i$ from the calibration dataset $\calibdataset$. For the residual $R = \max\left( \alpha_1 R^1 , \alpha_2 R^2, \cdots, \alpha_{n\horizon}R^{n\horizon} \right)$ sampled from the distribution $\distR
\in \mathcal{P}_{f,\tau}(\distzeroR)$, it holds that
    \[
    \Pr\left[ \bigwedge_{j=1}^{n\horizon} \left[ R^j \leq R^*_{\delta,\tau}/\alpha_j \right]\right] \geq \delta
    \]
    where $R^j$ is again the component-wise residual for $j\in [n\horizon]$.
\end{lemma}
\begin{proof}
The proof follows as the residual $R$ is the maximum of the scaled version of component-wise residuals so that
\[
 R = \max\left( \alpha_1 R^1 , \alpha_2 R^2, \cdots, \alpha_{n\horizon}R^{n\horizon} \right)\!\! \iff \!\! \bigwedge_{j=1}^{n\horizon} \left[ R^j \leq \frac{R}{\alpha_j} \right].
\]
Now, since $\Pr\left[ R \leq R^{*}_{\delta,\tau} \right] \geq \delta$ as well as $R<R^*_{\delta,\tau} \iff R^j< R^*_{\delta,\tau}/\alpha_j$ for all $j\in [n\horizon]$, we can claim that  $\displaystyle \Pr\left[ \bigwedge_{j=1}^{n\horizon} \left[ R^j \leq R^*_{\delta,\tau}/\alpha_j \right]\right] \geq \delta$   
\end{proof}
% We take advantage of this valuable feature of our selection for residual, proposed by Lemma \ref{lem:maxcontribution} and utilize it in the probabilistic reachability analysis. 

Next, we introduce the notion of an inflating zonotope to define the inflated flowpipe from the surrogate flowpipe.

\begin{definition}[Inflating Zonotope]\label{def:in_z} A zonotope $\zon(b,A)$ is defined
as a centrally symmetric polytope with $b \in \reals^k$ as its center, and $A =
\{ g_1,\ldots,g_p \}$ is a set of generators, where $g_i\in\reals^k$, that represents the set $\{ b + \mu_i g_i \mid \mu_i \in [-1,1]\}$.
Here, we introduce the inflating zonotope with base vector, 
\[
A = \mathbf{diag}\left(0_{1\times n},  \frac{R^*_{\delta,\tau}}{\alpha_1}, \cdots, \frac{R^*_{\delta,\tau}}{\alpha_{n\horizon}}\right),
\]
and center, $b$ is the vector $\mathbf{0}$ of length $(n+1)\horizon$; the notation $\mathbf{diag}(v)$ represents a diagonal matrix with the elements of $v$ along its diagonal and off-diagonal elements being zero. 
\end{definition}

 Including this inflating zonotope in our probabilistic reachability analysis leads to the following result.
\begin{theorem}\label{lem:inclusion_conformal}
Let $\bar{X}$ be a surrogate flowpipe of the surrogate model $\overallf$ for the set of initial conditions $\init$. Let $R^{*}_{\delta,\tau}$ be computed from the calibration dataset $\calibdataset$, as shown before. If we use $R^*_{\delta,\tau}$ to construct the inflated surrogate flowpipe,
\[
X = \bar{X} \oplus \zon(0, \mathbf{diag}([0_{1\times n}, e )),\quad e = \left[ R^*_{\delta,\tau}/\alpha_1, \cdots, R^*_{\delta,\tau}/\alpha_{n\horizon} \right],
\]
then it holds that $X$ is a $\delta$-confident flowpipe for any $\trajreal_{\statee_0} \sim \dist \in \mathcal{P}_{f,\tau}(\distzero)$ with $\statee_0 \stackrel{\mathcal{W}}{\sim} \init$.
\end{theorem}
\begin{proof}Assume again that $\trajreal_{\statee_0} \sim \dist \in \mathcal{P}_{f,\tau}(\distzero)$ with $\statee_0 \stackrel{\mathcal{W}}{\sim} \init$, and recall that, 
$$
R  = \max[\alpha_1 R^1, \hdots, \alpha_{n\horizon}R^{n\horizon}],
$$ 
where $ R^j =  \left|e_{j+n}^\top \trajreal_{\statee_0} - \mathsf{F}^j(\statee_0)\right|$. Applying Lemma \ref{lem:maxcontribution} results in
$
\Pr\left[  \bigwedge_{j=1}^{n\horizon} \left(R^j \leq R^{*}_{\delta,\tau}/\alpha_j \right) \right] \geq \delta.
$
% \begin{equation}\label{eq:0}
% \Pr\left[  \bigwedge_{j=1}^{n\horizon} \left(R^j \leq R^{*}/\alpha_j \right) \right] \geq \delta.
% \end{equation}
We rephrase this  as 
\[
\Pr\left[  \bigwedge_{j=1}^{n\horizon} \left( \mid e_{j+n}^\top \trajreal_{\statee_0} - \mathsf{F}^j(\statee_0) \mid \leq  R^{*}_{\delta,\tau}/\alpha_j \right) \right] \geq  \delta.
\]
% $$
% \Pr\left[  \bigwedge_{j=1}^{n\horizon} \left( \mid e_{j}^\top \traj_{\statee_0} - \mathsf{F}^j(\statee_0) \mid \leq  R^{*}/\alpha_j \right) \right] \geq  \Delta.
% $$
Next, we define the interval $C_j(\statee_0)$ as 
$$
C_j(\statee_0) := \left[ \mathsf{F}^j(\statee_0)  -R_{\delta,\tau}^{*}/\alpha_j\ ,\  \mathsf{F}^j(\statee_0)  +  R_{\delta,\tau}^{*}/\alpha_j \right]
$$ 
and accordingly obtain the guarantee that
\[
\Pr\left[  \bigwedge_{j=1}^{n\horizon} \left( e_{j+n}^\top \trajreal_{\statee_0} \in C_j(\statee_0) \right) \right]  \geq  \delta.
\]
% $$
% \Pr\left[  \bigwedge_{j=1}^{n\horizon} \left( e_{j}^\top \traj_{\statee_0} \in C_j(\statee_0) \right) \right]  \geq  \Delta
% $$
Based on this representation, we can now see that
\begin{equation} \label{eq:1}
\Pr\left[  \trajreal_{\statee_0} \in \zon\left([\statee_0^\top,\overallf(\statee_0\ ;\theta)] , \mathbf{diag}\left([0_{1\times n},\ e]\right)\right) \right]\geq \delta.
\end{equation}
Finally, since $\Pr[\statee_0 \notin \init] =0$ and $\bar{X}$ is a surrogate flowpipe for the surrogate model $\overallf$ on $\init$, i.e., $\statee_0 \in \init$ implies $[\statee_0^\top,\overallf(s_0\ ; \theta)] \in \bar{X}$, we can conclude,
\begin{equation}\label{eq:2}
\begin{aligned}
&\zon\left([\statee_0^\top,\overallf(\statee_0\ ;\theta)] , \mathbf{diag}\left([0_{1\times n},\ e]\right)\right)\\
&\hspace{2mm} \subset \bar{X} \oplus \zon\left(0 ,\mathbf{diag}\left([0_{1\times n},\ e]\right)\right) = X
\end{aligned}
\end{equation}
% \begin{equation}\label{eq:2}
% \zon\left(\overallf(\statee_0\ ; \theta) , \mathbf{diag}\left([0_{1\times n},\ e]\right)\right) \subset \bar{X} \oplus \zon\left(0 , \mathbf{diag}\left([0_{1\times n},\ e]\right)\right) = X
% \end{equation}
Consequently, we know that $\Pr[\trajreal_{\statee_0} \in X  ] \geq \delta$ holds.
\end{proof}
We note that the surrogate reachability, and also use of the Minkowski sum in the reachability analysis, results in some level of
conservatism.

% Consider the trajectory $\traj_{\statee_0}$ and its corresponding residual as $R$ in the mentioned proof for Theorem \ref{lem:inclusion_conformal}. In case, we utilize exact-star reachability analysis then $\traj_{\statee_0} \in X \iff R< R^*$. Therefore, the Minkowski sum  imposes no conservatism on the reachability analysis and doesn't affect  $\mathbf{Beta}$ distribution for possible deviations of marginal guarantee. 
% \begin{remark}
%     In this research, we do not analyze the impact of reachability analysis on the $\mathbf{Beta}$ distribution for $\Pr[R<R^*]$. However, the minkowski sum in Theorem \ref{lem:inclusion_conformal} makes it certain that for all $\delta \sim \mathbf{Beta}$ the confidency of reachset is greater than $\Delta=\delta$. Thus, even if it increases the variance of the  putterbation for $\del$, it will not negatively impact our provable probabilistic guarantee. 
% \end{remark}

\begin{remark}
We note that we can even compute the minimum size of the calibration
dataset required to achieve a desired confidence probability $\delta\in (0,1)$.
Robust conformal inference \cite{cauchois2020robust} imposes two constraints in this regard. The first
constraint specifies  a relation between the adjusted miscoverage level
$\bar{\epsilon}$ and the size of the calibration dataset as $\lceil
(L+1)(1-\bar{\epsilon}) \rceil \leq L$. The second constraint is that the ranges of $g_{f,\tau}$ and $g_{f,\tau}^{-1}$ have to be within $[0,1]$. Thus, we can impose
$(1+1/L)g_{f,\tau}^{-1}(\delta)<1$, or in other words $L> \lceil
g_{f,\tau}^{-1}(\delta) / (1- g_{f,\tau}^{-1}(\delta) ) \rceil$.
\end{remark}

\mypara{Tightening the surface area of the flowpipe}
The scaling factors $\alpha_j$ are trained to minimize the sum of errors over the
trajectory components, see equation \eqref{eq:L2}. The expression $R^*_{\delta, \tau} \sum_{j=1}^{n\horizon}1/\alpha_j$ arising from \eqref{eq:L2} can also be interpreted as the surface area of the inflating zonotope, see Definition \ref{def:in_z}. We now show how we can update scaling factors after training to reduce the surface area to tighten the $\delta$-confident flowpipe further.  Let us  sample a new trajectory dataset $\lpdataset$ 
% \footnote{We cannot utilize the calibration dataset defined before for this purpose as it makes the residual's definition dependent on the calibration dataset breaking the exchangeability assumption required for conformal inference.}.
and compute the prediction errors $R_i^j$ and residuals $R_i$ for $i\in[|\lpdataset|]$, and also their conformalized robust $\bar{\delta}$-quantile $R_{\delta,\tau}^*$, using the trained scaling factors $\alpha_j$ and surrogate model. 

The main idea for an efficient update of the trained scaling factors is as follows. Assume $\alpha^{'}_j$ is the updated version of $\alpha_j$. If this update is such that the updated trajectory residuals $ \max(\alpha^{'}_1 R_1^i, \cdots, \alpha^{'}_{n\horizon} R_{n\horizon}^i), i \in [|\lpdataset|]$ are the same as the trajectory residuals $R_i$ under $\alpha_j$, then $R_{\delta,\tau}^*$ under the updated $\alpha^{'}_j$ remains the same. By defining $\omega^{'}_j = 1/\alpha^{'}_j$, we see that the surface area  $R^*_{\delta, \tau} \sum_{j=1}^{n\horizon} \omega^{'}_j$ of the inflating zonotope depends linearly on $\omega^{'}_j$. On the other hand the constraint $R_i = \max \left(  R_i^1/\omega_1^{'}, \cdots, R_i^{n\horizon}/\omega_{n\horizon}^{'}\right),$ is a linear constraint. This constraint can be equivalently represented as
\begin{equation}
\forall i\in[|\lpdataset|], j\in[n\horizon] \ R_i \omega_j^{'} \geq R_i^j\nonumber
\end{equation}
under the additional assumption that the updated scaling factors $\omega^{'}_j$ are minimized.
This means an efficient update on scaling factors to reduce the surface area can be done via  linear programming with decision variables $\omega^{'}_j , j \in[n\horizon]$, i.e.,
\begin{equation}\label{eq:alphoptim}
\textbf{minimize}\ \sum_{j=1}^{n\horizon} \omega_j^{'}\ \  \textbf{s.t. }\ \  \forall i\in[|\lpdataset|], j\in[n\horizon] \ \omega_j^{'} \geq R_i^j/R_i,
\end{equation}
which has the analytical solution $\omega^{'}_j = \max_{i}\left[R_i^j/R_i\right]$.
% This implies our guarantee for the flowpipe is still valid. Therefore, we update the scaling factors, $\alpha_j = 1/ \omega_j, j \in [n\horizon]$ while the following relation,
% \begin{equation}\label{eq:previous}
% \forall i \in [|\lpdataset|],\ \  R_i = \max \left(  R_i^1/\omega_1^{'}, \cdots, R_i^{n\horizon}/\omega_{n\horizon}^{'}\right),
% \end{equation}
%  still holds, where $\omega_j^{'}$ is the updated version of $\omega_j$. This provides an opportunity to introduce a linear program with decision variables $\omega_j^{'}, j \in[n\horizon]$, to tighten the surface area of the zonotope as follows:
% \begin{equation}\label{eq:alphoptim}
% \left\{
% \begin{aligned}
% & \textbf{minimize}\ R_{\delta,\tau}^* \sum_{j=1}^{n\horizon} \omega_j^{'}\\
% & \quad \textbf{s.t.}\quad \forall i\in[|\lpdataset|], j\in[n\horizon] \ R_i \omega_j^{'} \geq R_i^j.
% \end{aligned}
% \right.
% \end{equation}
% This linear program captures the fact that the surface area of inflating
% zonotope is $\sum_{j=1}^{n\horizon} R_{\delta,\tau}^* \omega_j$ and for a given $i
% \in [|\lpdataset|]$ the relation in \eqref{eq:previous} implies that
% $\forall j\in[n\horizon] \ R_i \omega_j \geq R_i^j$\footnote{In case we also admit
% $R_{\delta,\tau}^*$ and $R_i, i\in[|\lpdataset|]$ as decision variables then,
% this optimization is neither linear nor convex and therefore, does not scale
% even for sampling a small set of data-points $R_i^j,
% i\in[|\lpdataset|],j\in[n\horizon]$.}.}
\section{Experimental Results}
\label{sec:results}
\navid{To mimic real-world systems that can produce actual
trajectory data, we use stochastic difference equation-based models derived from
dynamical system models. In these difference equations, we assume additive
Gaussian noise that models uncertainty in observation, dynamics, or even
modeling errors}. 

% To match the theory for discrete-time stochastic systems presented in this paper, we only evaluate the SDE at discrete time points.

\navid{Our theoretical guarantees depend on knowledge of the distribution shift $\tau$. In practice, however, $\tau$ is usually not known {\em a priori} but can be estimated from the data.  For the purpose of providing an empirical examination of our results, we fix $\tau$ {\em a priori} to compute the $\delta$-confident flowpipe and construct a system $\dist$ from $\distzero$ by varying system parameters such that $\distR \in \mathcal{P}_{f,\tau} (\distzeroR)$. } We ensure that this holds by estimating the distribution shift, denoted by $\tilde{\tau}$, as the $f$-divergence between $\distzeroR$ and $\distR$ and by making sure that $\tilde{\tau} \leq \tau$. In our experiments, we used the total variation distance for $f$, and used $3\times10^5$ trajectories to estimate $\tilde{\tau}$.

% In fact, from data processing inequality we know that we only need to make sure that $\distR \in \mathcal{P}_{f,\tau} (\distzeroR)$, where we recall that $\distzeroR$ and $\distR$ are the residual distributions of the simulated and the shifted system, respectively. 

\navid{We use ReLU activation functions in our surrogate NN-based models
motivated by recent advances in NN verification with ReLU activations. We
specifically use the NNV toolbox from \cite{tran2020nnv} for reachability
analysis of the surrogate model. While other activation functions could be used,
we expect more conservative results in case we utilize non-ReLU activation functions. The approach in
\cite{tran2020nnv} uses star-sets (an extension of zonotopes) to represent the
reachable set and employs two main methods: (1) the exact-star method that
performs exact but slow computations, (2) the approx-star method that is
conservative but faster. To mitigate the runtime of the exact-star technique and
the conservatism of the approx-star technique, set partitioning can be utilized
\cite{tran2019safety}, where initial states are partitioned into
sub-regions and reachability is done on each sub-region in
parallel.}

% \navid{We strongly suggest utilizing ReLU activation function for training the surrogate model. The recent achievements in the literature \cite{tran2020nnv} have provided low conservative results for ReLU-NN reachability analysis. Motivated with these achievements, we restricted our toolbox only to ReLU neural networks. In case we utilize other activation functions, we expect more conservative results, as the surrogate reachability may not scale well to the surrogate models we utilize in this section. The approach in \cite{tran2020nnv} uses star-sets (an extension of zonotopes) to represent the set of states reached at
% the output of each layer of the $\relu$-NN. The set of initial states (assumed to be some polyhedron) is input to the $\relu$-NN, and the output of the final layer represents the image of the set of initial states w.r.t. the $\relu$-NN. There are two main methods for doing this: (1) the exact-star method that performs precise computations but can be slow, (2) the approx-star method that over-approximates the flowpipe and is faster. To mitigate the runtime of the exact-star technique and the conservatism of the approx-star technique, set partitioning can be utilized \cite{tran2019safety}, where the initial states are partitioned into sub-regions and reachability analysis is conducted on each sub region in parallel.}

As per Theorem \ref{lem:inclusion_conformal}, our results are guaranteed to be valid with a confidence of $\delta$. To
determine how tight this bound is, we will  empirically examine the computed probabilistic flowpipes. We do so by sampling i.i.d. trajectories from $\dist$ \footnote{\navid{We use trajectories close to the worst case where $\tilde{\tau}$ is close to $\tau$.}} and computing the ratio of the trajectories that are included in the probabilistic flowpipes, which we denote by $\tilde{\Delta}$.
% The surrogate flowpipes (in turn used to construct probabilistic flowpipes per Theorem \ref{lem:inclusion_conformal}) are computed using the NNV toolbox \cite{tran2020nnv}. 
Additionally, to check the coverage guarantee $\delta$ for $R_{\delta,\tau}^*$ directly,
we also report the ratio of the trajectories that provide a residual less
than $R_{\delta,\tau}^*$, which we denote with $\tilde{\delta}$. We emphasize that
$\tilde{\Delta}$ and $\tilde{\delta}$ are both expected to be greater than
$\delta$.

% \footnote{We have mentioned in the preliminaries that the confidence level, $\delta$ proposed with conformal inference is subject to small perturbations ($\mathbf{Beta}$-distribution), regarding the calibration dataset that we sample. This perturbation can be neglected by incorporating a large enough calibration dataset. It is important to note that unlike the perturbations for $\delta$ which follows the $\mathbf{Beta}$ distribution, $\tilde{\Delta}$ may not, even if $\tilde{\Delta}$ is computed accurately. However, the theory of robust conformal inference makes it certain that the accurate estimation for $\tilde{\Delta}$ satisfies $\tilde{\Delta}>\delta$. This implies, the perturbation distribution for the accurate estimation of $\tilde{\Delta}$ can not violate our proposed guarantee for the confidence of flowpipe.}

% \footnote{Since the examination addresses a specific example of $\dist \in \mathcal{P}_{\tau,f} (\distzero)$, then $\tilde{\delta}$ (regarding the selection of calibration dataset) may not follow the $\mathbf{Beta}$ distribution even if it is estimated accurately. However, the theory of robust conformal inference makes it certain that the accurate estimation for $\tilde{\delta}$ satisfies $\tilde{\delta}>\delta$. This implies, the perturbation distribution for the accurate estimation of $\tilde{\delta}$ can not violate our proposed guarantee for the confidence of $R^*$.}

In the remainder, we first present a case study to compare between reachability with surrogate models using the mean square error (MSE) and our proposed quantile loss function in \eqref{eq:mainloss}. We show that the quantile loss function results in tighter probabilistic flowpipes. After that, we present several case studies on a $12$-dimensional quadcopter and the time reversed van Der Pol dynamics. The results are also summarized in Table~\ref{tbl}. \navid{We visualize our flowpipes by their two-dimensional projection. Therefore, in case a trajectory is included in all the visualized bounds, it does not necessary mean the trajectory is covered. We instead, determine the inclusion of traces in our star-sets using the NNV toolbox which determines set inclusion by solving a linear programming feasibility problem.}

\mypara{Comparison between MSE and Quantile minimization}
\myipara{Experiment 1} Our first experiment will show the advantage of training a surrogate model with quantile loss function compared to training a surrogate model using the MSE loss function. 
% \footnote{ The prediction error is generated from two different sources. The first source is the inaccuracy of the trained model and the second source is the inherent stochasticity in the dataset. Although the latter is not trainable we can reduce the impact of the former source by increasing the quality of training. In case the dataset is generated from a straight forward dynamics, like a linear model, then the MSE can provide prediction errors with very small variance  (regarding the former source). This implies the average and $\delta$-quantile of prediction error (from the former source) are close enough and leaves no room to show the contribution of training with quantile minimization.}
Therefore, we model $\distzero$ as the non-linear system 
\[
\begin{aligned}
&x_{\timeid+1} = 0.985y_{\timeid}+\sin(0.5x_{\timeid})-0.6\sin(x_{\timeid}+y_{\timeid})-0.07+0.01v_1\\
&y_{\timeid+1} = 0.985x_{\timeid}+\cos(0.5y_{\timeid})-0.6\cos(x_{\timeid}+y_{\timeid})-0.07+0.01v_2
\end{aligned}
\]
that generates a periodic motion. Here, $v_1$ and $v_2$ denote random variables sampled from a normal distribution. In this experiment, we do not consider a shifted stochastic system $\dist$, and instead sample trajectories from $\distzero$ for comparison of our two surrogate models. The first surrogate model is trained as proposed in Section \ref{sec:surrogate} using  quantile minimization, while the other surrogate model is trained with the MSE loss function. Our results are shown upfront in  Figure \ref{fig:MSEquantile} where we compare the probabilistic reachable sets of these two models. 

In more detail, recall that the scaling factors $\alpha_1,\hdots,\alpha_{nK}$ of our proposed method in  Section \ref{sec:surrogate} are jointly trained with the surrogate model. However, since we do not train these scaling factors jointly when we use the MSE loss function, we instead compute them beforehand following \cite{zhao2023robust}. In other words, we normalize the component-wise residuals as
\[
\alpha_j = 1/ \omega_j \text{ where } \omega_j = \max \left( R_1^j,\ R_2^j,\ \hdots,\  R_{|\traindataset|}^j \right) .
\]
for each $j\in[n\horizon]$. We utilized $|\traindataset| = 10^5$ random trajectories with $\horizon = 50$ for training the surrogate model.  The initial states were uniformly sampled from the set of initial states $\init_1 = [-0.5, 0.5] \times [-0.5, 0.5]$.  In both case, we trained a $\relu$ surrogate model with structure $[2,\ 20,\ 50,\ 90,\ 100]$ and we applied approx-star from the NNV \cite{tran2020nnv} toolbox for the reachability analysis. To lower the conservatism of approx-star, we partition the set of initial states into $400$ partitions, and perform the surrogate reachability analysis for every partition separately. The flowpipe is also computed for the confidence level of $ \delta \geq 95\%$. The details of the experiment via quantile minimization are also provided in Table \ref{tbl}. 

\navid{We additionally compare the surface area $R^*_{\delta, \tau} \sum_{j=1}^{n\horizon} 1/\alpha_j$ of the inflating zonotopes, see Definition \ref{def:in_z},  for both surrogate models. \navid{Note that this surface area is the $\loss_2$ loss in equation \eqref{eq:L2} when $q=R^*_{\delta, \tau}$, which we enforce during training. The $\bar{\delta}$-quantile of $\mathbf{UB}_i$ as defined in \eqref{eq:UB} is the $\loss_2$ loss, and hence approximates the surface area of the inflating zonotope. To compare the distributions of $\mathbf{UB}_i$,}
% Note that the $95\%$-quantile of the surface area is approximated by the loss function $\loss_2$, for which we compute an upper bound, denoted by $\mathbf{UB}$, as per equation ~\eqref{eq:UB}.
 we simulate $3\times 10^5$ trajectories and compute $\mathbf{UB}_i/(n\horizon)$ for both the MSE  and the quantile loss-based NNs. We present the histograms of $\mathbf{UB}_i/(n\horizon)$ for both loss functions in Figure \ref{fig:MSEquantile_error} where we see that the quantile of $\mathbf{UB}_i$  for MSE is larger. This emphasizes the advantage of training via quantile loss function.}

% \begin{figure}[t]
%     \centering
%     \begin{subfigure}{\linewidth}
%     \centering
%         \includegraphics[width = 0.7\linewidth]{figures/result_density2.eps}
%         \captionsetup{skip=0pt}
%         \caption{Flowpipe for $x_\timeid$ and $y_\timeid$ over time steps. The red borders are for flowpipes generated by MSE loss function and the blue ones are for quantile based loss function. The shaded region shows an approximation of flowpipe by recording trajectories, and the darkness of the green color shows the density of the trajectories. The black lines are the borders for the shaded region. The shaded area is generated via $300000$ trajectories.} 
%     \label{fig:MSEquantile}
%     \end{subfigure}
%     % \vfill
%     \begin{subfigure}{\linewidth}
%     \centering
%         \includegraphics[width = 0.55\linewidth]{figures/comparison_MSE_q2.eps}
%         \captionsetup{skip=0pt}
%         \caption{Distribution of $\mathbf{UB}/(n\horizon)$ for the MSE and the quantile-based NNs for $3\times 10^5$ samples. The $95\%$-quantile of variable $\mathbf{UB}/(n\horizon)$ represents the surface area of the obtained inflating zonotope. The figure is cropped for better visibility.}  
%     \label{fig:MSEquantile_error}
%     \end{subfigure}
%     \captionsetup{skip=0mm} \caption{Figures (a) and (b) show a comparison between flowpipes and distributions of $\mathbf{UB}/(n\mathrm{K})$ respectively for training via MSE and training via our proposed loss function \eqref{eq:mainloss}.}
% \end{figure}

\begin{figure}[t]
    \centering
        \includegraphics[width = 0.7\linewidth]{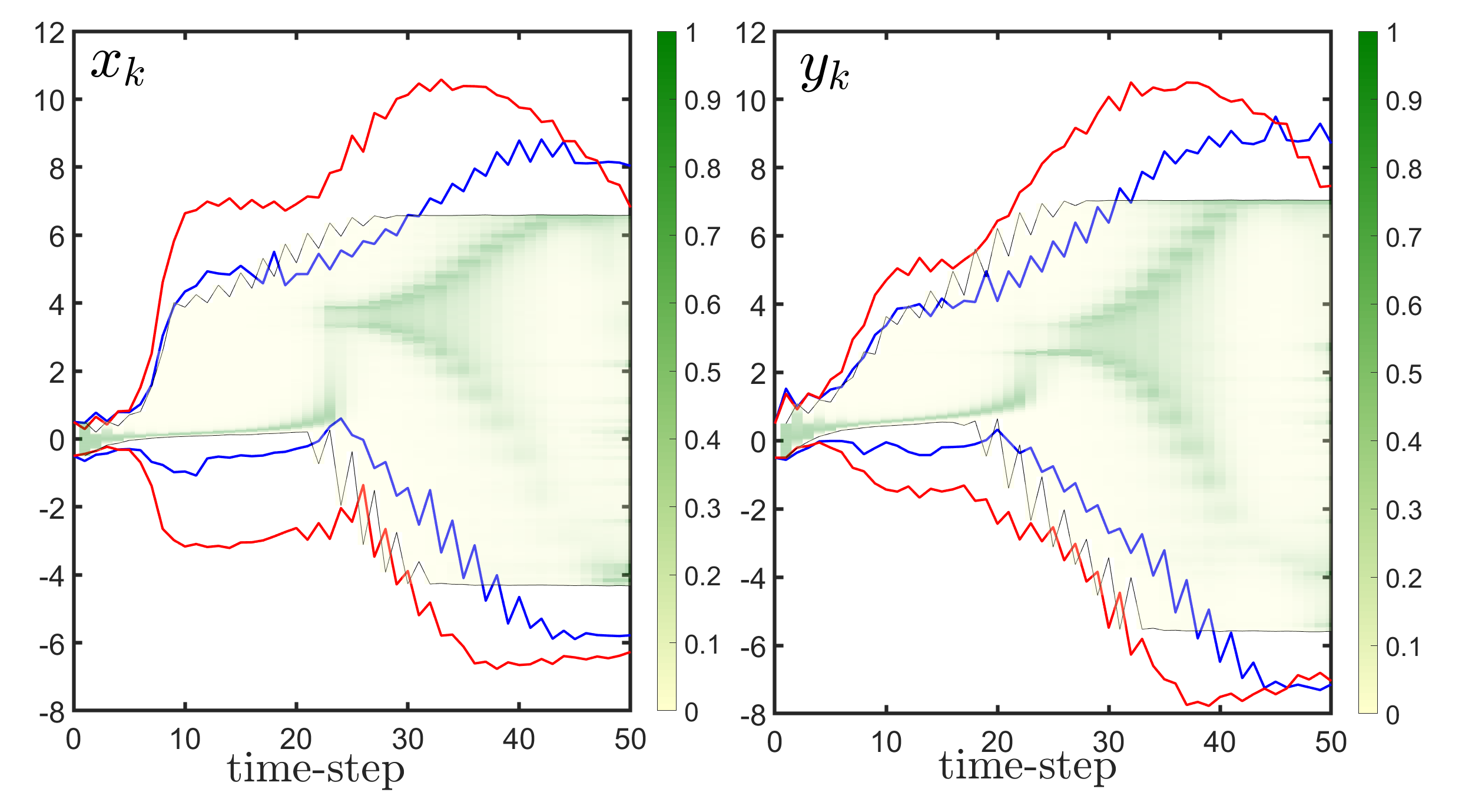}
        \captionsetup{skip=0pt}
        \caption{Flowpipe for $x_\timeid$ and $y_\timeid$ over time steps. The red borders are for flowpipes generated by MSE loss function and the blue ones are for quantile based loss function. The shaded region shows an approximation of flowpipe by recording trajectories, and the darkness of the green color shows the density of the trajectories. The black lines are the borders for the shaded region. The shaded area is generated via $300000$ trajectories.} 
    \label{fig:MSEquantile}
\end{figure}
    % \vfill
\begin{figure}[t]
    \centering
        \includegraphics[width = 0.55\linewidth]{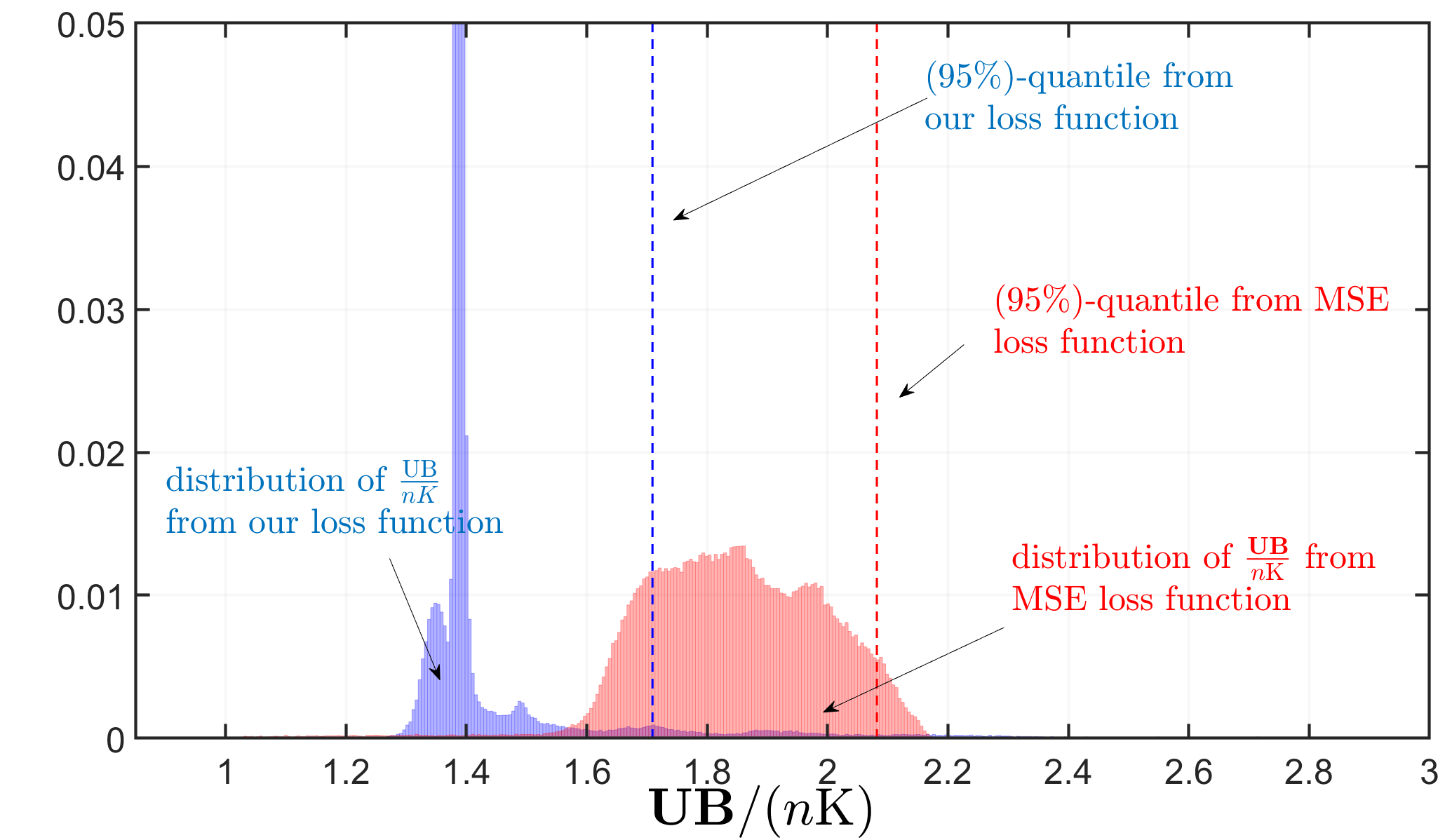}
        \captionsetup{skip=0pt}
        \caption{Distribution of $\mathbf{UB}/(n\horizon)$ for the MSE and the quantile-based NNs for $3\times 10^5$ samples. The $95\%$-quantile of variable $\mathbf{UB}/(n\horizon)$ represents the surface area of the obtained inflating zonotope. The figure is cropped for better visibility.}  
    \label{fig:MSEquantile_error}
    % \captionsetup{skip=0mm} \caption{Figures (a) and (b) show a comparison between flowpipes and distributions of $\mathbf{UB}/(n\mathrm{K})$ respectively for training via MSE and training via our proposed loss function \eqref{eq:mainloss}.}
\end{figure}

\begin{table*}
\hspace{-15mm}
\resizebox{0.75\textwidth}{!}{
\begin{tabular*}{\textwidth}{@{\extracolsep{\fill}}lcc||ccccc}
\cmidrule{1-8}
\muc{3}{Specification} & \muc{5}{Reachability Analysis with Robust CI}\\
\cmidrule{1-3}\cmidrule{4-8}
    Experiment $\#$: & Confidency    & Maximum      &   Number of    &  Conformal & Reachability    & Overal    & Size of   \\
    Underlying   & of flowpipe,  & distribution &   Star-sets    &  inference      & technique & reachability & calibration\\
% \cmidrule{8-9}\cmidrule{10-11}
     dynamics    &  i.e. $\delta$& shift's radius   &    from NNV     &  runtime($\sec$) & & runtime($\sec$)  & dataset ($|\calibdataset|$) \\
\cmidrule{1-8}
1:\ Periodic    & $\delta = 95\%$   & $0$    & $400$ & $0.0892$  & approx-star &$22.5233$ & $10,000$ \\
2:\ Quadcopter  & $\delta = 99.99\%$& $0$    & $64$  & $1.4299$  & approx-star & $ 160.0486$ & $20,000$ \\
3:\ Quadcopter  & $\delta = 80\%$   & $0.15$ & $64$  & $0.4971$  & approx-star & $ 148.9815$ & $10,000$ \\
4:\ Quadcopter  & $\delta = 70\%$   & $0.25$ & $64$  & $0.4971$  & approx-star & $ 148.9815$ & $10,000$ \\
% 4:\ Quadcopter  & $\delta = \%71$   & $0.24$ & $64$    & $1.4393$  & approx-star & $154.6004$ & $48.0056$ \\
% 5:\ TRVDP & $\delta = 89.5\%$   & $0.1$& $27$  & $0.0977$  & exact-star & $ 1.0277$ & $10,000$\\
5:\ TRVDP & $\delta = 99.99\%$& $0$    & $1$   & $4.8218$  & exact-star & $ 1.5761$ & $30,000$ \\
6:\ TRVDP & $\delta = 77\%$   & $0.225$& $1$   & $0.2876$  & exact-star & $ 0.1404$ & $10,000$ \\
\cmidrule{1-8}
\end{tabular*}}\\\\\\
\hspace{-5mm}
\resizebox{0.80\textwidth}{!}{
\begin{tabular*}{\textwidth}{@{\extracolsep{\fill}}lcccc}
\cmidrule{1-5}
& \muc{4}{Training} \\
\cmidrule{2-5}
         &  Training  & Size of training   & Linear programming &   Size of dataset  for \\
         &  runtime   & dataset ($|\traindataset|$)    &  runtime           &   linear programming ($|\lpdataset|$) \\
\cmidrule{1-5}
Experiment 1:   & $41$ minutes    & $100,000$ & $1.7422$ \ seconds  & $10,000$\\
Experiment 2:   & $124$ minutes   & $40,000$  & $29.3255$ seconds & $2,000$\\
Experiment 3,4: & $112$ minutes   & $40,000$  & $21.1406$ seconds & $2,000$\\
% Experiment 5:   & $46$  minutes   & $100,000$  & $2.0934$ \ seconds  & $10,000$\\
Experiment 5:   & $25$  minutes   & $40,000$  & $6.3559$ \ seconds  & $50,000$\\
Experiment 6:   & $27$  minutes   & $40,000$  & $2.8236$ \ seconds  & $10,000$\\
\cmidrule{1-5}
\end{tabular*}}
\vspace{5mm}
\centering
\hspace{-5mm}
\resizebox{0.80\textwidth}{!}{
\begin{tabular*}{\textwidth}{@{\extracolsep{\fill}}lccc|cc}
\cmidrule{1-6}
 & \muc{5}{Examination} \\
\cmidrule{2-6}
    &  Example of induced  &  \muc{2}{Coverage Estimation (i.e. $\tilde{\Delta}$) } & \muc{2}{Coverage Estimation (i.e. $\tilde{\delta}$)} \\
     &  distribution shift's&  \muc{2}{for flowpipe generated by:}  & \muc{2}{for $R_{\delta,\tau}^*$  generated by:} \\
% \cmidrule{8-9}\cmidrule{10-11}
       &  radius (i.e. $\tilde{\tau})$ &  Robust CI & Vanilla CI          &  Robust CI & Vanilla CI          \\
\cmidrule{1-6}
Experiment 1:  & $0$     & $ 96.31\%$& $96.31\%$& $ 95.05\%$ & $95.05\%$\\
Experiment 2:  & $0$     & $100\%$& $100\%$& $99.99\%$ & $99.99\%$\\
Experiment 3:  & $0.1445$& $100\%$& $100\%$& $88.58\%$ & $70.86\%$ \\
Experiment 4:  & $0.2395$& $100\%$& $100\%$& $80.50\%$ & $49.64\%$ \\
% Experiment 5:  & $0.0838$& $96.86\%$& $90.71\%$& $94.61\%$ & $85.05\%$ \\
Experiment 5:  & $0$     & $99.99\%$& $99.99\%$& $99.99\%$ & $99.99\%$\\
Experiment 6:  & $0.2085$& $95.91\%$& $56.52\%$& $95.87\%$ & $55.73\%$ \\
\cmidrule{1-6}
\end{tabular*}}
\caption{Shows the detail of our computation process to provide probabilistically guaranteed flowpipes. \navid{ The time horizon for experiments 1,5,6 is $\horizon=50$ time-steps and for the experiments 2,3,4 is $\horizon = 100$ time-steps. The sampling time for quadcopter and TRVDP are $0.05$ and $0.02$ seconds, respectively.} We examine the results with a valid distribution shift (explained in detail in Table \ref{tbl:extended}) that is less than the maximum specified distribution shift in terms of total variation. This shift is estimated through the comparison between $300,000$ trajectories from $\dist$ and $\distzero$. We also utilize $10,000$ trajectories (number of trials) from this specific distribution $\dist$ to examine the coverage of flowpipes and $300,000$ trajectories for examination of the coverage level for $R_{\delta, \tau}^*$ (i.e.$ \tilde{\Delta},\ \tilde{\delta}$). To evaluate the contribution of robust conformal inference, we also solve for the flowpipes again neglecting the distribution shift, i.e. $\bar{\epsilon} = \epsilon$, and show the coverage guarantee for $R_{\delta, \tau}^*$ and flowpipes may get violated, ($\tilde{\delta}< \delta$ or $\tilde{\Delta}< \delta$), in case the shifted distribution (deployment distribution) is considered. \navid{The runtimes we report for reachability assumes no parallel computing.}}
\label{tbl}
\end{table*}
\newcommand{\uni}[1]{\mathsf{uni}(#1)}
\begin{table}[t]
\centering
\resizebox{0.86\textwidth}{!}{
\begin{tabular*}{\linewidth}{@{\extracolsep{\fill}}lccl}
\toprule
& Experiment. & Distribution.  &  \muc{1}{$\Sigma$ for added noise} \\ 
& & & \muc{1}{Gaussian $\gaussian(\mathbf{0}, \Sigma)$} \\
\midrule
\multirow{6}{3.5cm}{$\distzero$} &
1 & $\uni{\init_1}$ & $\mathbf{diag}([0.01, 0.01])^2$\\
& 2 & $\uni{\init_2}$ & $\mathbf{diag}([0.05\cdot\vec{1}_{1 \times 6},0.01\cdot\vec{1}_{1 \times 6}])^2$\\
& 3 & $\uni{\init_2}$ & $\mathbf{diag}([0.05\cdot \vec{1}_{1 \times 6},0.01\cdot \vec{1}_{1 \times 6}])^2$ \\
& 4 & $\uni{\init_2}$ & $\mathbf{diag}([0.05\cdot \vec{1}_{1 \times 6},0.01\cdot \vec{1}_{1 \times 6}])^2$ \\
& 5,6 & $\uni{\init_3}$ & $\mathbf{diag}([0.1,0.1])^2$ \\
\midrule
\multirow{4}{3.5cm}{$\dist \in \mathcal{P}_{\tau,f} (\distzero)$} &
3 &  $\uni{\init_2}$ & $\Sigma \times 1.8$ \\
& 4 & $\uni{\init_2}$ & $\Sigma \times 2.2$ \\
& 6 & $\uni{\init_3}$ & $\mathbf{diag}([0.1378,0.1378])^2$ \\
\bottomrule
\end{tabular*}}
\caption{Initial state distribution and added Gaussian noise (mean:$0$, covariance:$\Sigma$)  for the training and the shifted environments; $\uni{\init}$ denotes the uniform distribution over $\init$.}
\label{tbl:extended}
\end{table}

\begin{figure}[t]
\hspace{-8mm}
\includegraphics[width =1.08\linewidth]{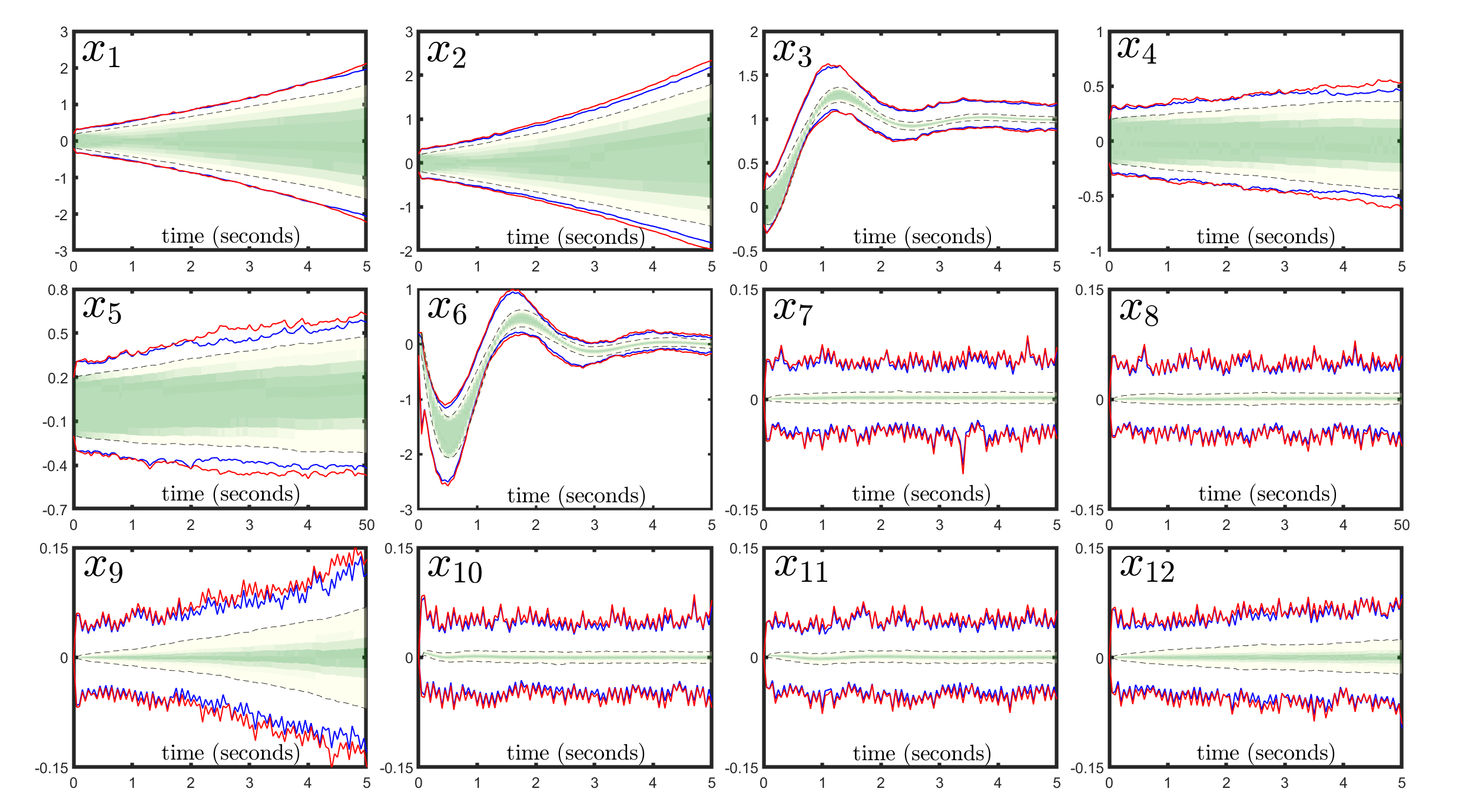}
\caption{This figure shows the proposed flowpipes computed for the quadcopter dynamics for each state component over the time horizon of $100$ time steps with $\delta t = 0.05$ that means $5$ seconds operation of quadcopter. The red borders show the flowpipe that contains trajectories from $\distzero$ with provable coverage of $\delta \geq  99.99\%$. The green shaded area shows the density of a collection of $300,000$ of these trajectories, and the darker color means the higher density of traces. The blue borders are also for a flowpipe that contains the trajectories from distribution $\distzero$ with $ \delta \geq 95\%$. The dotted black line also shows the border of collected simulated trajectories.}
\label{fig:quad_comparison}
\end{figure}

\begin{figure*}
    \centering
    \begin{subfigure}{0.42\textwidth}
        \includegraphics[width=\linewidth]{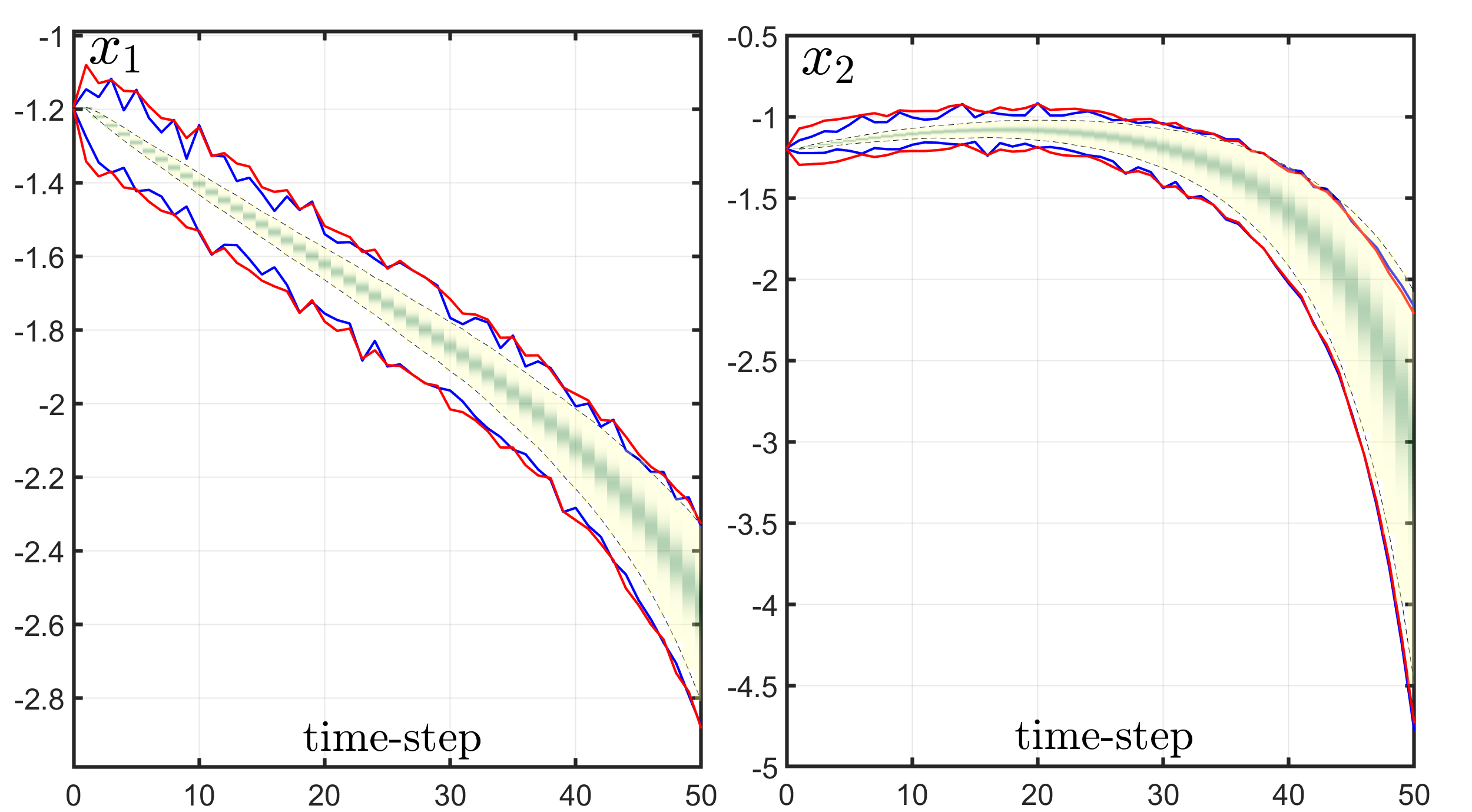}
        \caption{ $(\tau,\ \delta) = (0,\ 0.9999)$  }
        \label{fig:vdp1}
    \end{subfigure}
    \hfill
    \begin{subfigure}{0.415\textwidth}
        \includegraphics[width=\linewidth]{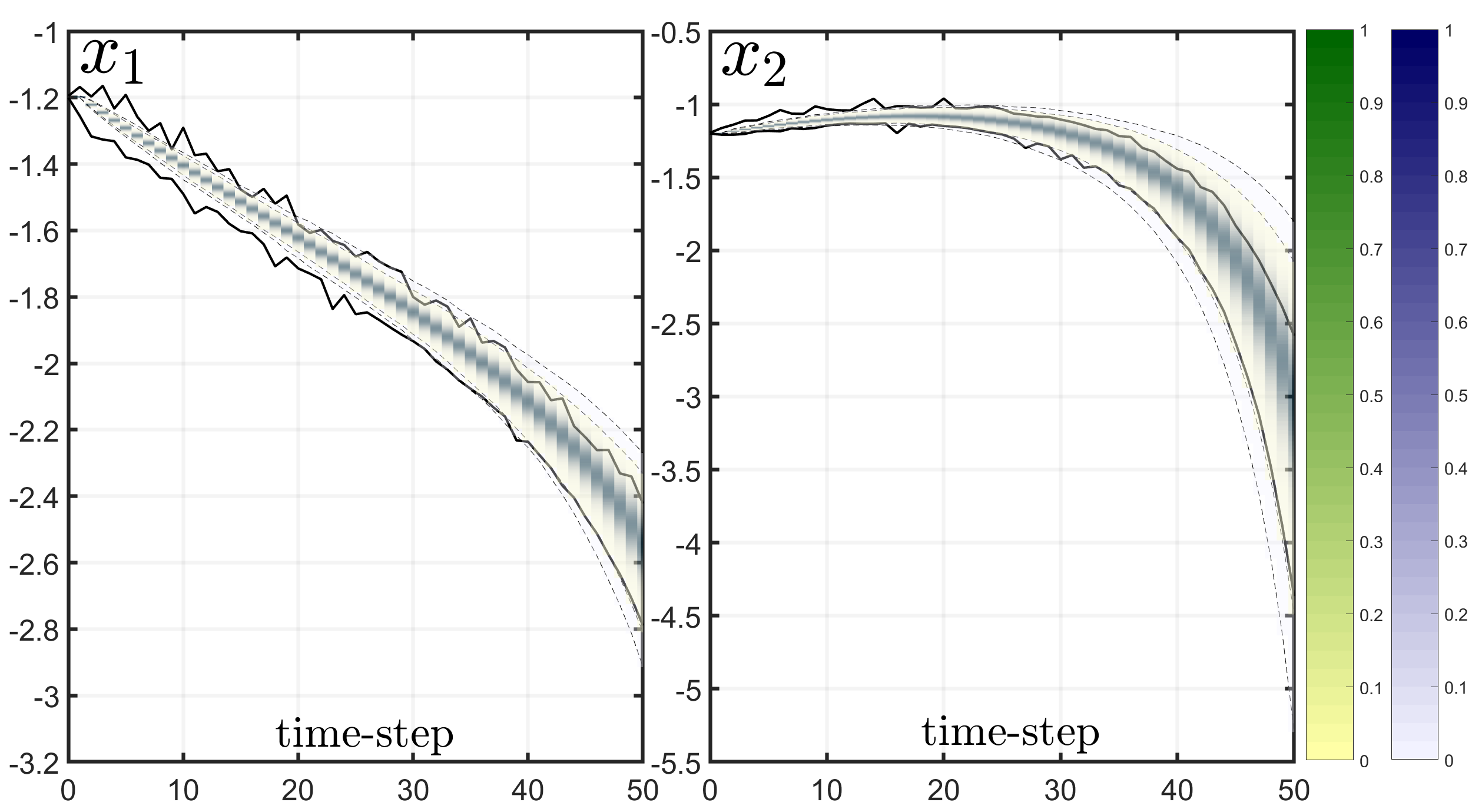}
        \caption{$(\tau,\ \delta) = (0.225,\ 0.77)$}
        \label{fig:vdp2}
    \end{subfigure}
    \hfill
    \begin{subfigure}{0.145\textwidth}
        \includegraphics[width=\linewidth]{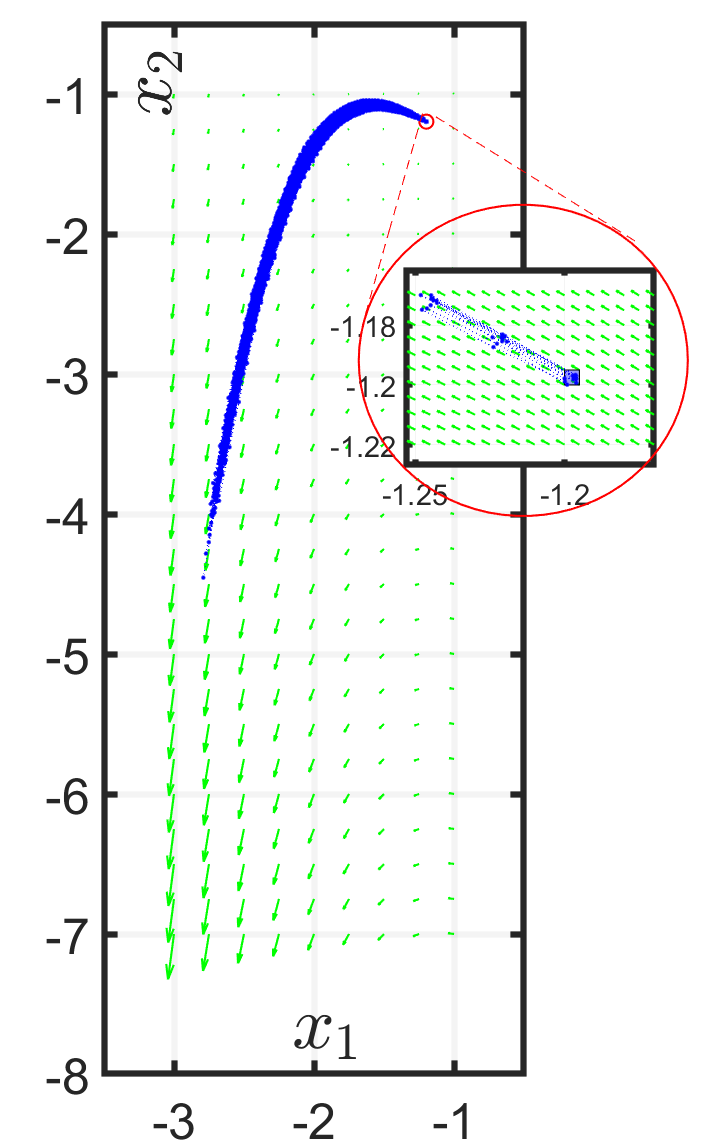}
        \caption{}
        \label{fig:vdp2_demo}
    \end{subfigure}
    \caption{Shows the density of trajectories starting from $\init_3$ versus their computed flowpipes. The green color-bar represents the density of traces from, $\distzero$ and the blue color-bar is for traces from $\dist$. The shaded areas are generated via $3\times10^5$ different trajectories, and the dotted lines represents their border. a) Shows two different flowpipes for TRVDP dynamics with confidence level of $0.9999$ on $\distzero$. The tighter flowpipe (blue color) utilizes the linear programming \eqref{eq:alphoptim} while the looser one (red color) does not. b) Shows a flowpipe that covers trajectories from $\dist$ with the confidence level of $77\%$ and also covers the traces from $\distzero$ with the confidence level of $99.5\%$. The blue shaded area is for $\dist$ and the green shaded area is for $\distzero$. c) Shows the vector field of TRVDP dynamics that illustrates the instability of the system. }
    % Shows the bounds from the flowpipe for TRVDP dynamics and one selected calibration dataset. The yellow shaded area represents the trajectories from shifted distribution, and the green shaded area is for original distribution. However, regarding the clarity of the plots, this is important to note, the green area is plotted on top of the yellow area. The dynamics is chaotic to the extent the projection of flowpipe on $x_2$ is still inside the green area (from original distribution) while its confidence level is $\delta> 99.5\%$. and $ x_1(0),x_2(0)\in[-1.2,\ -1.195]$}}
\end{figure*}

\mypara{$12$-Dimensional Quadcopter}
Next, we consider a $12$ dimensional quadcopter model from the benchmarks in \cite{huang2022polar} that is designed to hover around a pre-specified elevation. The ODE model for this system is provided as follows:

\begin{equation}
\begin{array}{ll}
\dot{x}_1 = & \cos(x_8) \cos(x_9)x_4  + (\sin(x_7) \sin(x_8) \cos(x_9)- \cos(x_7) \sin(x_9))x_5\\
            & + (\cos(x_7) \sin(x_8) \cos(x_9)+ \sin(x_7) \sin(x_9))x_6 + v_1\\
\dot{x}_2 = &\cos(x_8) \sin(x_9)x_4 + ( \sin(x_7)  \sin(x_8)  \sin(x_9)+ \cos(x_7)  \cos(x_9)) x_5\\
            & +( \cos(x_7)  \sin(x_8)  \sin(x_9)- \sin(x_7)  \cos(x_9)) x_6+v_2\\
\dot{x}_3 = & \sin(x_8) x_4- \sin(x_7)  \cos(x_8) x_5- \cos(x_7)  \cos(x_8) x_6+v_3\\
\dot{x}_4 = & x_{12} x_5-x_{11} x_6-9.81  \sin(x_8)+v_4\\
\dot{x}_5 = & x_{10} x_6-x_{12} x_4+9.81  \cos(x_8)  \sin(x_7)+v_5\\
\dot{x}_6 = & x_{11} x_4-x_{10} x_5+9.81  \cos(x_8)  \cos(x_7)-9.81-u_1/1.4+v_6\\
\dot{x}_7 = & x_{10}+( \sin(x_7) ( \sin(x_8)/ \cos(x_8))) x_{11} + ( \cos(x_7) ( \sin(x_8)/ \cos(x_8))) x_{12}+v_7\\
\dot{x}_8 = & \cos(x_7) x_{11}- \sin(x_7) x_{12}+v_8\\
\dot{x}_9 = & ( \sin(x_7)/ \cos(x_8)) x_{11}+( \cos(x_7)/ \cos(x_8)) x_{12}+v_9\\
\dot{x}_{10} = & -0.9259 x_{11} x_{12}+18.5185 u_2+v_{10}\\
\dot{x}_{11} = & 0.9259 x_{10} x_{12}+18.5185 u_3+v_{11}\\
\dot{x}_{12} = & v_{12},
\end{array}
\end{equation}
\navid{where the state consists of the position and velocity of the quadrotor $x_1,x_2,x_3$ and $x_4,x_5,x_6$, respectively, as well as the Euler angles $x_7,x_8,x_9$, i.e., roll, pitch, and yaw, and the angular velocities $x_{10},x_{11},x_{12}$. Here the initial set of states is defined as,
$\init_2 = \{\statee_0 \mid i \in [1,6]: -0.2 \le \statee_0(i) \le 0.2, i \ge 7: \statee_0(i) = 0 \}$.}
We also add additive noise to the system that is detailed in Table \ref{tbl:extended}, and we generate data with time step \navid{$\delta t = 0.05$ seconds over $100$ time steps (i.e. $5$ seconds)}. The controller is a neural network controller that was presented in \cite{huang2022polar}. We present $3$ experiments on this model. \navid{Learning a surrogate model to map the $12$-dimensional initial state to a $1200$-dimensional trajectory is impractical. We thus use an interpolation technique to resolve this issue. To that end, we select only certain time-steps of the $1200$-dimensional trajectory in order to map the initial state to state values at the selected time steps, while we take care of the remaining time steps via interpolation. If the trajectories are smooth, as is the case in this case study, this is expected to work well. We here select every second time-step  to extract a $600$-dimensional trajectory  $(\delta t =0.1, \horizon =50)$ to train a surrogate model of structure $[12, 200, 400, 600]$. Finally we interpolate the sampled $600$-dimensional trajectory to approximate the original $1200$-dimensional trajectory  ($\delta t =0.05, \horizon =100$). This interpolation process is integrated in the model in an analytical way, and is done by multiplying a weight matrix, $W \in \mathbb{R}^{1200\times 600}$  to the last layer. This converts the model's structure to $[12, 200, 400, 1200]$ which will be utilized for the surrogate reachability. The scaling factors $\omega_j,j\in[n\horizon]$ will be also interpolated for un-sampled time-steps after the training and before the linear programming. }

\myipara{Experiment 2} In comparison with \cite{hashemi2023data}, we provide a higher level of data efficiency. Consider a confidence level of $99.99\%$, and no distribution shift. We assume a calibration dataset of size \navid{$|\calibdataset| = 2\times 10^4$} to compute $R_{\delta, \tau}^*$ and the $\delta$-confident flowpipe, and a $\relu$ neural network of structure $[12,\ 20,\ 400,\ 1200]$ to train the surrogate model. The methodology proposed in \cite{hashemi2023data} requires a calibration dataset of at least \navid{$24\times 10^6$} data-points\footnote{ Minimum data size in \cite{hashemi2023data} is $|\calibdataset| > \lceil \frac{1+\gamma}{1-\gamma} \rceil$, where $\gamma = 1- \frac{1-\delta}{n\horizon}$.} to provide the mentioned level of confidence. On the other hand, we only require $10^4$ trajectories. Fig.~\ref{fig:quad_comparison} shows the proposed reach set and Table~\ref{tbl} presents the detail of the computation process. Our estimation shows that we achieve $\tilde{\delta}= 0.9999$ via $3\times 10^5$ trials and $\tilde{\Delta}= 1$ via $10^4$ trials, which aligns with our expectations.

\myipara{Experiments 3, 4} \navid{In this case study, we generate a $95\%$ confident flowpipe for the trajectories from $\distzero$ and we utilize it to study the distribution shift on two different deployment environments $\dist$. This flowpipe is  plotted in Figure \ref{fig:quad_comparison} and the details of the computation process is included in Tables \ref{tbl} and \ref{tbl:extended}. For this generated flowpipe, given a maximum distribution shift radius $\tau \in[0,1]$, the flowpipe's confidence level $\delta$ for trajectories from $\dist$ has to satisfy $\delta \geq 0.95 -\tau$. The bound $\delta \geq \bar{\delta}-\tau$ can  be derived from equation \eqref{eq:linearrule}.  Therefore, we consider two different scenarios. In  Experiment 3, we examine our flowpipe for the case $\tau = 0.15$. In this case, for a deployment environment with distribution shift, $\tilde{\tau} < 0.15$ we numerically show that $\tilde{\Delta}, \tilde{\delta} > 0.95-0.15 = 0.8$. In addition, in  Experiment 4, we assume $\tau = 0.25$ and for a deployment environment with $\tilde{\tau} < 0.25$ we show that $\tilde{\Delta}, \tilde{\delta} > 0.95-0.25 = 0.7$. Tables \ref{tbl} and \ref{tbl:extended} show the detail of the experiments and distribution shift respectively.}

\mypara{Time-reversed van Der Pol Oscillator Dynamics}
The time-reversed van Der Pol (TRVDP) dynamics is known for its inherent instability, which makes it a pernicious challenge for computing reach sets. The SDE model for TRVDP is:
\[ 
\begin{bmatrix} \dot{x}_1 & \dot{x}_2 \end{bmatrix}^\top = \begin{bmatrix} x_2 & \mu x_2 (1-x_1^2) -x_1\end{bmatrix}^\top+v,\quad \mu =-1,
\]
here, $v$ is an additive Gaussian noise, detailed in Table~\ref{tbl:extended}. We generate data from this dynamics with sampling time $\delta t = 0.02$ seconds, and we target reachability for $\horizon=50$ time step. We use a limited set of initial states  $\init_3=\left\{\statee_0 \mid [-1.2, -1.2] \leq \statee_0 \leq [-1.195, -1.195]\right\}$ to investigate the instability of the system dynamics. Our analysis centers on discerning how this instability manifests as a divergence in trajectories originating from this restricted set of initial states. We also assume a model with structure $[2,\ 50, \ 90, \ 100]$ to train the surrogate model. We perform two experiments on this system, explained below.

% \myipara{Experiment 5} We assume the set of initial state to be $\init_3=\left\{\statee_0 \mid [-1.2, -1.2] \leq \statee_0 \leq [-1.0, -1.0]\right\}$ and plan to generate a flowpipe that covers at least $91\%$ of trajectories even if they are shifted with $\tau < 0.085$ that is measured via total variation. We use a $\relu$ neural network with structure $[2,\ 50, \ 90, \ 100]$ to train the surrogate model. The specification and the details of reachability analysis and system's stochasticity are all presented in Tables~\ref{tbl} and \ref{tbl:extended}. Figure \ref{fig:vdp1} demonstrates the resulting flowpipe.

\myipara{Experiment 5} In this experiment, we target the flowpipe computation for the TRVDP dynamics for the confidence probability of $\delta \ge 99.99\%$ and no distribution shift. Figure \ref{fig:vdp1} shows the resulting flowpipe and Table~\ref{tbl} shows the details of the process. In this experiment, we also generate another $0.9999$-confident flowpipe excluding the linear programming (proposed in equation \eqref{eq:alphoptim}) from the process. Figure~\ref{fig:vdp1} also compares these flowpipe and shows removing the linear programming increases the level of conservatism.

\myipara{Experiment 6}  
We target \navid{an arbitrary} confidence level of $\delta \geq 0.77$ for the flowpipe, despite distribution shifts within radius $\tau< 0.225$ measured in total variation. As suggested by robust conformal inference, we should target a flowpipe with confidence level of $99.5\% = 77\% + 22.5\%$ on $\distzero$ to ensure the confidence level of $77\%$ on $\dist$. Figure \ref{fig:vdp2} shows our probabilistically guaranteed flowpipe, and Tables \ref{tbl},\ref{tbl:extended} present the detail of the experiment. These tables also show that, in case we set $\bar{\epsilon} = \epsilon$ in reachability analysis (Vanilla CI) then our flowpipe, violates the guarantee (i.e. $\delta \geq 0.77$). This emphasizes on the contribution of robust conformal inference.

\section{Conclusion}
\label{sec:conclusion}
% Robust conformal inference has been recently developed and proposed in notable works like
% \cite{cauchois2020robust,tibshirani2019conformal}. The robust form of conformal
% inference has shown promising results in various studies, as seen in
% \cite{zhao2023robust,gupta2021s,si2023pac,yilmaz2022test}. However, in this
% research similar to the research study \cite{zhao2023robust}, our primary focus
% lies in the methodology proposed in \cite{cauchois2020robust}. 

\mypara{Conclusion} This paper addresses challenges in data-driven reachability
analysis for stochastic dynamical systems, specifically focusing on 
distribution shifts between training and test environments. By
leveraging a dataset of $\horizon$-step trajectories, the approach constructs a
probabilistic flowpipe, ensuring that the probability of trajectory violation
remains below a user-defined threshold even in the presence of distribution
shifts. We propose the reliable guarantees with higher data-efficiency compared
to the existing techniques assuming knowledge of an upper bound for distribution shift. The methodology relies on three key principles:
surrogate model learning, reachability analysis using the surrogate model, and
robust conformal inference for probabilistic guarantees. We illustrated the
efficacy of our approach via reachability analysis on high-dimensional systems
like a $12$-dimensional quadcopter and unstable systems like the time-reversed van Der Pol oscillator.

%\mypara{Future work } We have defined the residuals as $R_i = \max(\alpha_1 R_i^1, \cdots, \alpha_{n\horizon}R_i^{n\horizon} ), i\in[|\calibdataset|]$ and we then trained the scaling factors $\alpha_j, j\in[n\horizon]$. In this case, the inflating zonotope is simply a hyper-rectangle. But if we define residual as $R_i = \max(\sum_{j=1}^{n\horizon} \alpha_{1,j}R_i^j, \cdots, \sum_{j=1}^{n\horizon} \alpha_{n\horizon,j}R_i^j  ), i\in[|\calibdataset|]$ and then train the squared matrix, containing parameters, $\alpha_{p,q} , p,q\in[n\horizon]$ then our inflating zonotope is a general zonotope which can reduce the level of conservatism noticeably. We consider this update in the future work. 

% 
% and leverage robust conformal inference to quantify prediction
% uncertainty using calibration data from $\distzero$. 
% \item Inspired by the research work \cite{cleaveland2023conformal}, we utilize the $\max()$ operator for residual normalization. 
% 
% W
% \item  We propose a linear programming approach that enhances the accuracy of reachability analysis while adhering to the principles of conformal inference.
% 
% 
% 
% \end{itemize}
% 

\section*{Acknowledgments}
This work was supported by the National Science Foundation through the following grants: CAREER award (SHF-2048094), CNS-1932620,  FMitF-1837131, CCF-SHF-1932620, the Airbus Institute for Engineering Research, and funding by Toyota R\&D and Siemens Corporate Research through the USC Center for Autonomy and AI.

\bibliographystyle{IEEEtran}
\bibliography{sample-base}
% \addtolength{\textheight}{-2cm}   % This command serves to balance the column lengths
%                                   % on the last page of the document manually. It shortens
%                                   % the textheight of the last page by a suitable amount.
%                                   % This command does not take effect until the next page
%                                   % so it should come on the page before the last. Make
%                                   % sure that you do not shorten the textheight too much.
\end{document}